\documentclass[letterpaper,twocolumn, journal]{IEEEtran}

\usepackage{fullpage,graphicx,url,setspace}
\usepackage{subfigure}
\usepackage{xparse}
\usepackage{stmaryrd}
\usepackage{cite}
\usepackage{bm}
\usepackage{color}
\usepackage[small,bf]{caption}
\usepackage{amsmath,verbatim,amssymb,amsfonts,amscd, graphicx, algorithmic, algorithm}
\usepackage{amsmath,amsthm}
\usepackage{mathtools}
\mathtoolsset{mathic}
\usepackage{microtype}
\usepackage{wrapfig}

\usepackage{ifxetex,ifluatex}
\ifxetex
\usepackage{xltxtra}
\else
\ifluatex
\usepackage{luatextra}
\else
\usepackage{amssymb}
\usepackage{algorithmic, algorithm}
\usepackage[utf8]{inputenc}

\csname else\expandafter\endcsname
\romannumeral -`0
\fi
\fi
\iftrue\relax%
\defaultfontfeatures{Ligatures=TeX}
\usepackage[math-style=TeX, vargreek-shape=TeX]{unicode-math}
\fi

\NewDocumentCommand{\entropy}{om}{\mathbb{H}\left[#2
    \IfValueT{#1}{\,\middle|\,#1}\right]}
\NewDocumentCommand{\bentropy}{lm}
  {\widetilde{\mathbb{H}}#1\left[#2\right]}

\NewDocumentCommand{\mutualInfo}{omm}{\mathbb{I}\left[#2;#3
    \IfValueT{#1}{\,\middle|\,#1}\right]}

\usepackage{tikz}
\usetikzlibrary{shapes.geometric, positioning, shadows}
\usepackage{graphicx,color}

\newtheorem{theorem}{Theorem}

\newtheorem{lemma}{Lemma}

\newtheorem{remark}{Remark}

\usepackage{enumitem}
\newlist{enumerate*}{enumerate*}{1}
\setlist[enumerate*]{label=(\arabic*)}

\newcommand{\ben}{\begin{eqnarray}}

\newcommand{\een}{\end{eqnarray}}

\newcommand{\transpose}{^{\intercal}}

\title{Online Supervised Subspace Tracking}

\author{Yao Xie$^*$, \quad \and Ruiyang Song, \quad \and Hanjun Dai, \quad \and Qingbin Li,  \thanks{Yao Xie$^*$ (Email: {yao.xie@isye.gatech.edu})
   is with the H. Milton Stewart School of
    Industrial and Systems Engineering, Georgia Institute of
    Technology, Atlanta, GA.
      }
    \quad \and Le Song
}

\begin{document}
\maketitle

\begin{abstract}
We present a framework for supervised subspace tracking, when there are two time series $x_t$ and $y_t$, one being the high-dimensional predictors and the other being the response variables and the subspace tracking needs to take into consideration of both sequences. It extends the classic online subspace tracking work which can be viewed as tracking of $x_t$ only. Our online sufficient dimensionality reduction (OSDR) is a meta-algorithm that can be applied to various cases including linear regression, logistic regression, multiple linear regression, multinomial logistic regression, support vector machine, the random dot product model and  the multi-scale union-of-subspace model. OSDR reduces data-dimensionality on-the-fly with low-computational complexity and it can also handle missing data and dynamic data. OSDR uses an alternating minimization scheme and updates the subspace via gradient descent on the Grassmannian manifold. The subspace update can be performed efficiently utilizing the fact that the Grassmannian gradient with respect to the subspace in many settings is rank-one (or low-rank in certain cases). The optimization problem for OSDR is non-convex and hard to analyze in general; we provide convergence analysis of OSDR in a simple linear regression setting. The good performance of OSDR compared with the conventional unsupervised subspace tracking are demonstrated via numerical examples on simulated and real data. 
\end{abstract}


\section{Introduction}
\label{introduction}


Subspace tracking plays an important role in various signal and data processing problems, including blind source separation \cite{blindSource2001}, dictionary learning \cite{DictionarySapiro09,wrightDictionary2012}, online principal component analysis (PCA) \cite{GROUSE1,robustPCA13},  imputing missing data \cite{GROUSE1}, denoising \cite{WangTu13}, and dimensionality reduction \cite{subspaceDim2014}.
%
To motivate online supervised subspace tracking, we consider online dimensionality reduction. Applications such as the Kinect system generate data that are high-resolution 3D frames of dimension 1280 by 960 at a rate of 12 frames per second. At such a high rate, it is desirable to perform certain dimensionality reduction on-the-fly rather than storing the complete data. In the unsupervised setting, dimensionality reduction is achieved by PCA, which projects the data using dominant eigenspace of the data covariance matrix. However, in many signal and data processing problems, 
side information is available in the form of labels or tasks. For instance, the data generated by the Kinect system contains the gesture information (e.g. sitting, standing, etc.) \cite{BiswasBasu2011,KinectHand2012}.
A \emph{supervised dimensionality reduction} may take advantage of the side information in the choice of the
subspaces for dimensionality reduction. %
The supervised dimensionality reduction is a bit more involved as it has two objectives: making a choice
of the subspace that represents the predictor vector and choosing parameters for the
 model that relates the predictor and response variables.

Existing online subspace tracking research has largely focused on
unsupervised learning, including the GROUSE algorithm (based on online
gradient descent on the Grassmannian manifolds)
\cite{GROUSE1,GROUSE2,GROUSEproof}, the PETRELS algorithm \cite{PETRELS} and
the MOUSSE algorithm \cite{MOUSSE2013}. Local convergence of GROUSE
has been shown in \cite{GROUSEproof} in terms of the expected
principle angle between the true and the estimated subspaces.
A preliminary exploration for supervised subspace tracking is reported
in \cite{logistic2012}, which performs dimensionality reduction
on the predictor vector \emph{without} considering the response variable in
the formulation.
%
%
What can go wrong if we perform subspace tracking using only the predictor $\{x_t\}$ but ignoring the response variable $\{y_t\}$ (e.g., the approach in \cite{logistic2012})? Fig. \ref{fig:illustration} and Fig. \ref{fig:bad_idea} demonstrate instances in classification, where unsupervised dimensionality reduction using a subspace may completely fail while the supervised dimension reduction does it right.


\begin{figure}
\begin{center}
\begin{tabular}{cc}
\includegraphics[width = 0.45\linewidth]{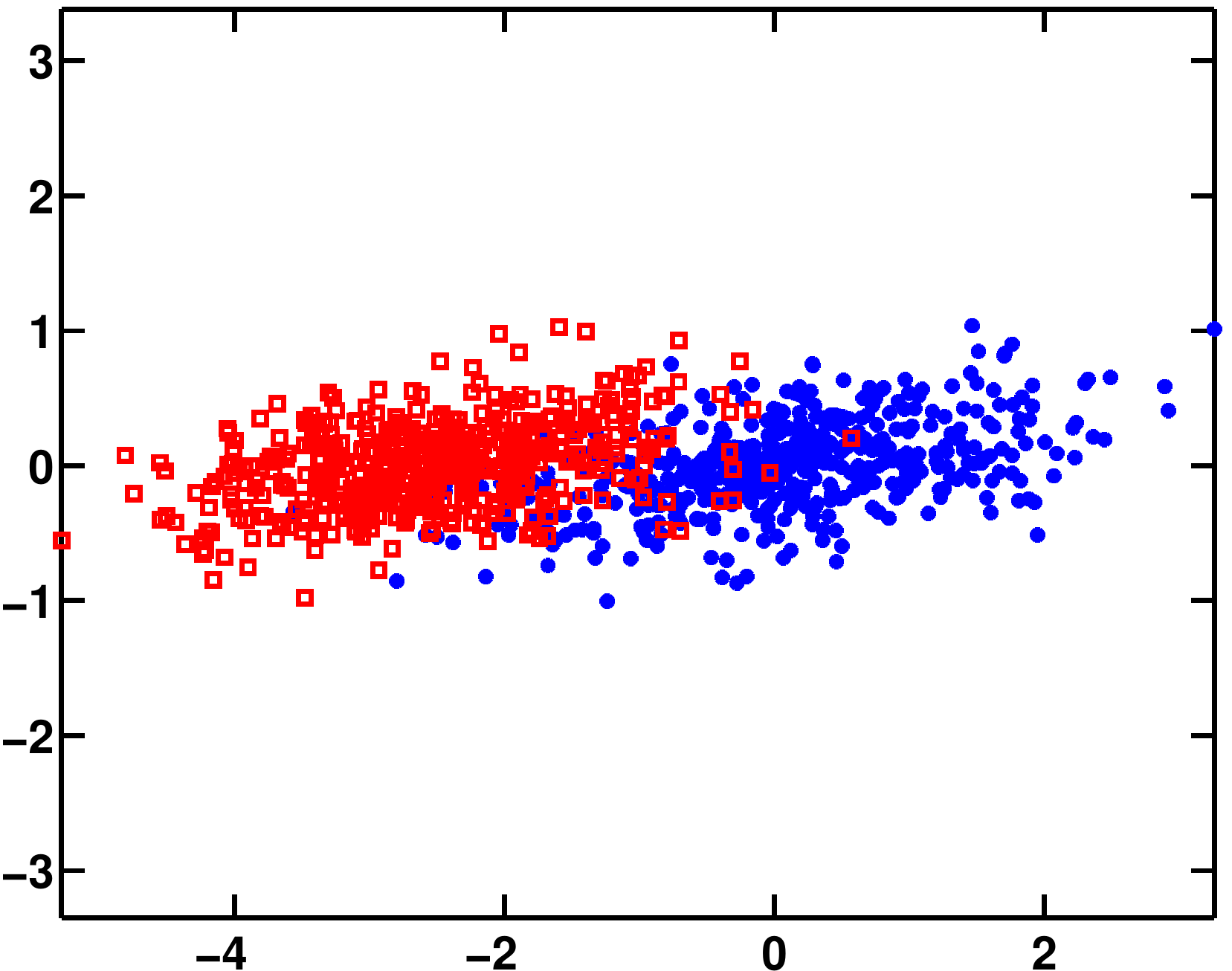} &
\includegraphics[width = 0.45\linewidth]{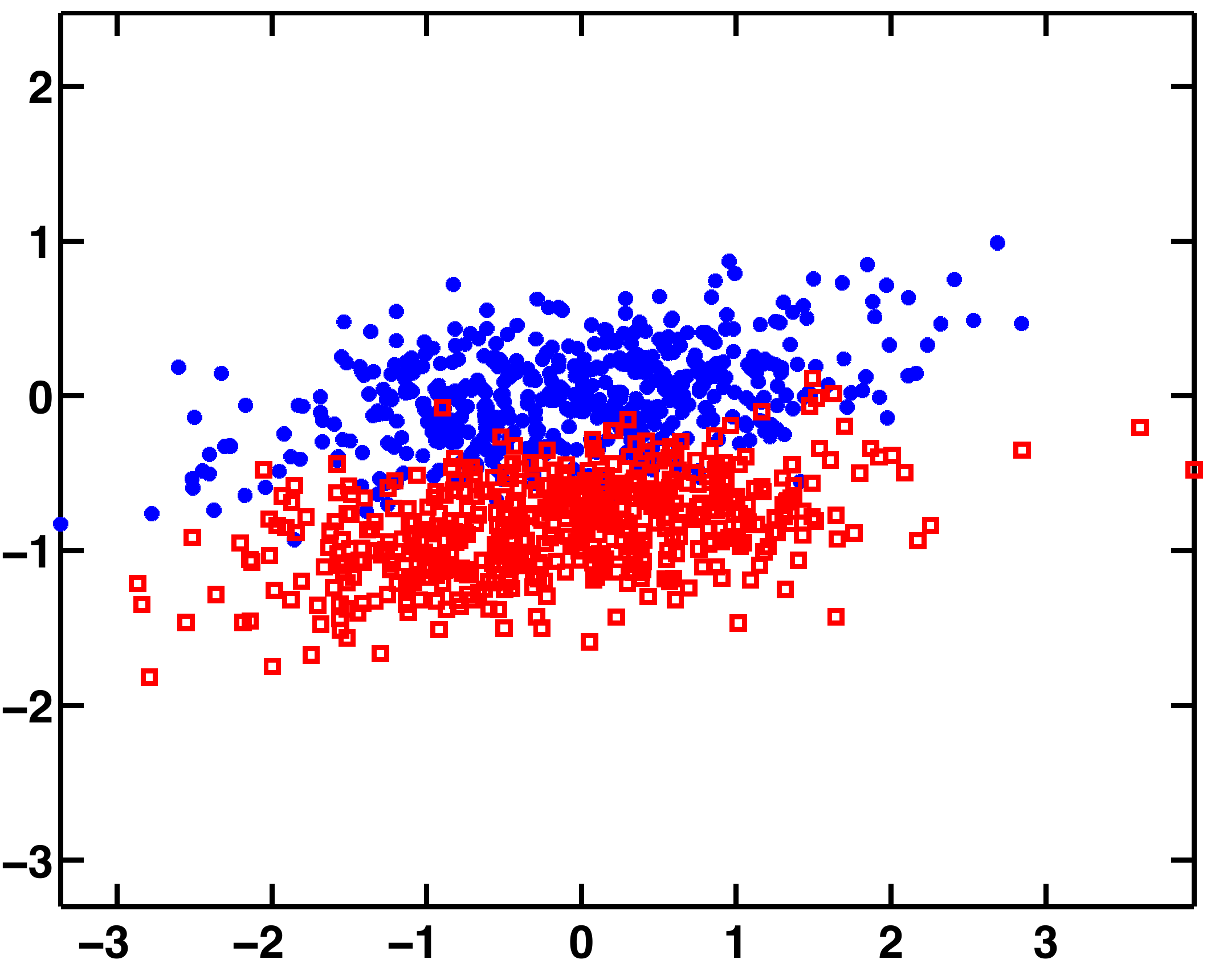}\\
(a) & (b)
\end{tabular}
\end{center}
\caption{Example to illustrate the difference of unsupervised and supervised subspace tracking. (a): a case where dimensionality reduction along largest eigenvector is optimal and both supervised and unsupervised dimensionality reduction  pick the dominant eigenvector; (b): a case where dimensionality reduction along the second largest eigenvector is optimal. Unsupervised dimensionality reduction erroneously picks the largest eigenvector since it only maximizes the variance of the predictor variable, but the supervised dimensionality reduction, by considering the response variable, correctly picks the second largest eigenvector.}
\label{fig:illustration}
\end{figure}

\begin{figure}[ht]
\begin{center}
 \includegraphics[width = 0.45\linewidth]{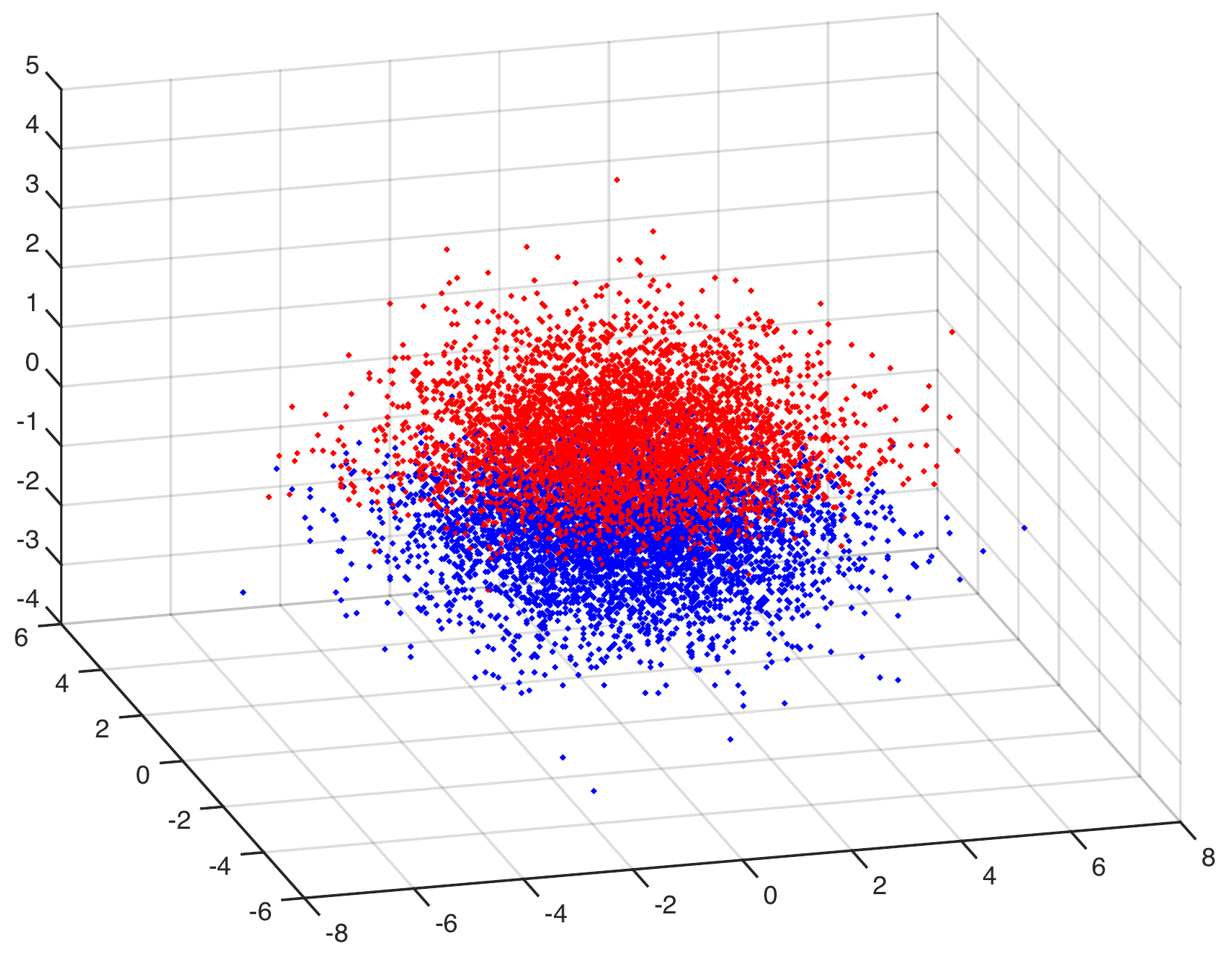} \\
 (a)
 \end{center}
 \begin{center}
\begin{tabular}{cc}
        \includegraphics[width = 0.45\linewidth]{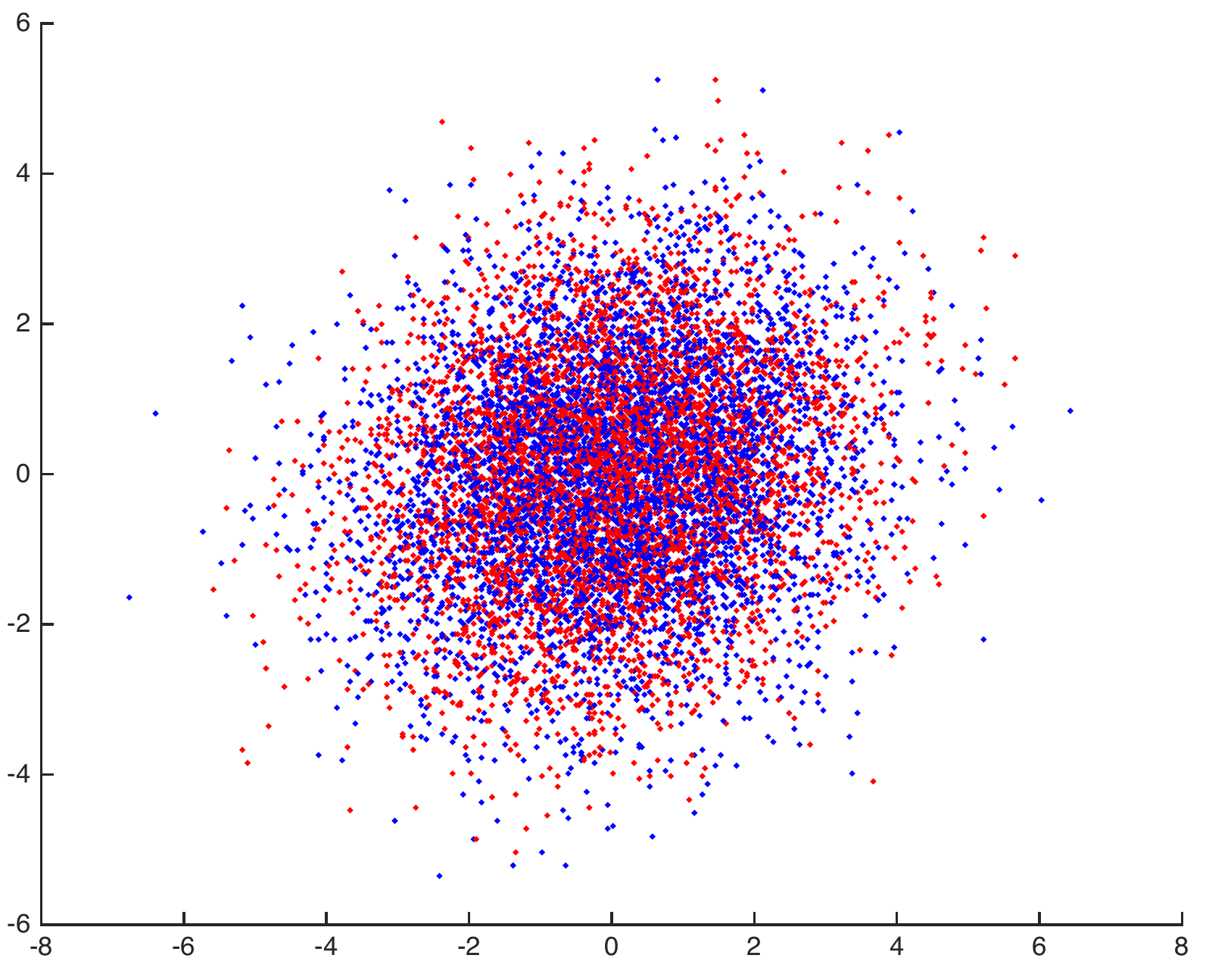} &
        \includegraphics[width = 0.45\linewidth]{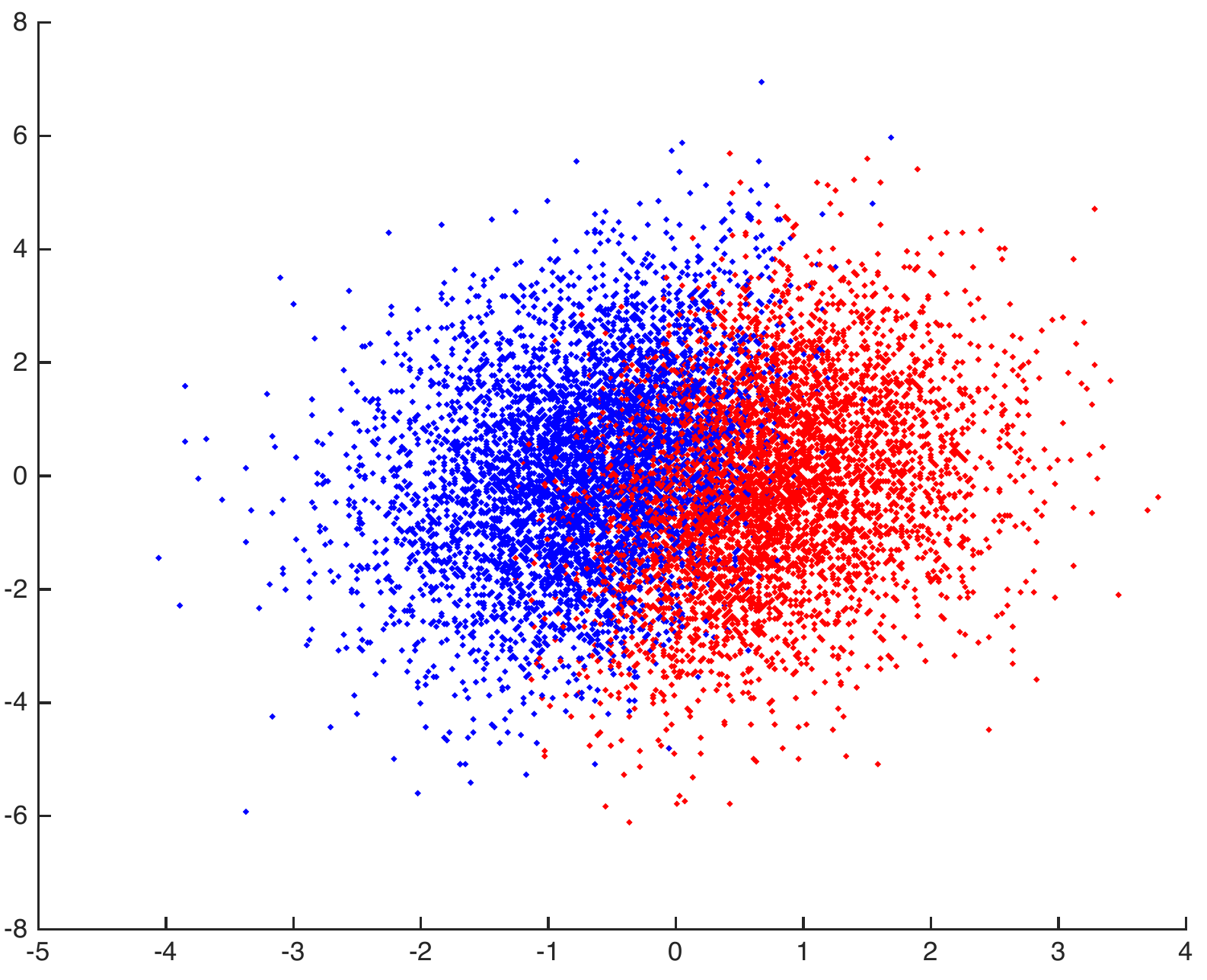} \\
        (b) & (c)
\end{tabular}
\end{center}
\caption{Simulated data correspond to the case in Fig.\ref{fig:illustration}(b): (a): three-dimensional data cloud $D = 3$ with two classes; (b): scatter plot of the data projected by unsupervised subspace tracking into a two-dimensional space $d =2$; the two classes are completely overlapped in the projected space; (c): scatter plot of the data projected by supervised  subspace tracking into a two-dimensional space $d = 2$; the two classes are well-separately. More details for this example can be found in Section \ref{sec:num_eg}.
}
\label{fig:bad_idea}
\end{figure}

In this paper, we present a general framework for supervised subspace tracking which we refer to as the online supervised dimensionality reduction (OSDR), which is a meta-algorithm that may be applied to various models. OSDR simultaneously learns the subspace and a predictive model through alternating minimization, and the formulation of OSDR takes into consideration both the high-dimensional predictor sequence $\{x_t\}$ and the response sequence $\{y_t\}$.
We explore different forms of OSDR under various settings, with the loss function being induced by linear regression, logistic regression, multiple linear regression, multinomial logistic regression, support vector machine, random dot-product model, and union-of-subspaces model, respectively. A common structure is that the Grassmannian gradient of the cost function with respect to the subspace $U$ is typically rank-one or low-rank (e.g., rank-$k$ for the $k$-classification problem, or the rank being dependent on the number of samples used for a mini-batch update). This structure enables us to develop a simple and efficient update for $U$ along the geodesic. Due to the orthogonality requirement and bilinearity, the optimization problem involved in OSDR is non-convex. We provide convergence analysis for OSDR in a simple linear regression setting. Good performance of OSDR is demonstrated on simulated and real data.

A notion in statistics related to our problem is sufficient dimensionality reduction, which
combines the idea of dimensionality reduction with the
concept of sufficiency. Given a response variable $y$, a
$D$-dimensional predictor vector $x$, a dimensionality reduction statistic $R(x)$ is sufficient if the distribution of $y$ conditioned on $R(x)$ is the same as that of $y$
conditioned on $x$. In other words, in the case of sufficiency, no information about the
regression is lost by reducing the dimension of $x$.
Classic sufficient dimensionality reduction methods include the sliced
inverse regression (SIR) \cite{SIR1991}, which uses the
inverse regression curve, $\mathbb{E}[x|y]$ to perform a weighted
principle component analysis; more recent works \cite{CookForzani2009} use likelihood-based
sufficient dimensionality reduction in estimating the central subspace. From this perspective, OSDR can be viewed as aiming at sufficiency for online subspace based dimensionality reduction.
In the offline setting, a notable work is sufficient dimension
reduction on manifold \cite{NilssonShaJordan2007}, which considers
the problem of discovering a manifold that best preserves information
relevant to a non-linear regression using a convex optimization
formulation. 

The rest of the paper is organized as follows. Section \ref{sec:formulation} sets up the formalization and the meta-algorithm form for OSDR.  Section \ref{sec:OSDR_models} presents specific OSDR algorithms under various settings. Section \ref{sec:theory} includes theoretical analysis. Section \ref{sec:num_eg} contains numerical examples based on simulated and real data. 


The notation in this paper is standard: $\mathbb{R}_+$ denotes the set of positive real numbers; $\llbracket n \rrbracket =\{1,2,\ldots,n\}$; $(x)_+ = \max\{x,0\}$ for any scalar $x$; $[x]_j$ denotes the $j$th element of a vector $x$; $\mathbb{I}\{\varepsilon\}$ is the indicator function for an event $\varepsilon$; $\|x\|_1$ and $\|x\|$ denote the $\ell_1$ norm and $\ell_2$ norm of a vector $x$, respectively; $X\transpose$ denotes transpose of a vector or matrix $X$ and $\|X\|$ denotes the Frobenius norm of a matrix $X$; $I_d$ is the identity matrix of dimension $d$-by-$d$. Define the sign function $\mbox{sgn}(x)$ is equal to 1 if $x > 0$ and is equal to 0 if $x < 0$.

\section{OSDR: the meta-algorithm}\label{sec:formulation}

Assume two time-series $(x_t, y_t)$, $t = 1, 2, \ldots$, such that $y_t$ can be predicted from the high-dimensional vector $x_t \in \mathbb{R}^D$. Here $y_t \in \mathcal{Y}$ can be either binary, real, or vector-valued.
To reduce the data dimensionality, we project $x_t$ by a subspace $U_t \in \mathbb{R}^{D\times d}$ to $x_t$, with $d\ll D$, to obtain a projected vector (or feature) $U_t\transpose x_t$. Here $d$ is the number of principle components we will use to explain the response series. Ideally, we would like to choose an $U_t$ that maximizes the prediction power of $U_t\transpose x_t$ for $y_t$, and $U_t$ can be time-varying to be data-adaptive. With such a goal, we formulate the \emph{online supervised dimensionality reduction} (OSDR), which simultaneously learns the subspace $U_t$ and a function that relates $x_t$ and $y_t$, by minimizing a cost function which measures the quality of the projection in terms of predictive power. The optimization problem of minimizing the loss function is inherently non-convex, due to the orthogonality requirement for $U_t$:  $U_t\transpose U_t = I_d$, as well as bilinear terms arising from coupling of $U_t$ and the model parameters. Hence, we develop the algorithm based on an alternating minimization and stochastic gradient descent scheme.

We consider two related formulations: the $d$-formulation, which assumes the response function only depends on the projected components and hence is a compact model that fits into the goal of dimensionality reduction. However, the $d$-formulation cannot deal with missing data. Then we also introduce the $D$-formulation, which handles missing data or can also be used for denoising. The $D$-formulation estimates the projection $\beta$ of the data using the subspace from the previous step, and then uses $\beta$ to update the subspace; however, it requires us to store high-dimensional model parameters. The loss function for the $d$- and the $D$-formulations are different, but the Grassmannian gradient with respect to $U$ is often rank-one (or low-rank). Such simple structure enables us to derive efficient algorithm to update $U$ along the geodesic, as summarized in Algorithm \ref{OSDR}. In the following derivations, we omit the time indices $t$ for notational simplicity.

\subsection{$d$-formulation}

The $d$-formulation is motived by sliced inverse regression, which assumes that the response variable $y$ depends only on the projected components. The loss function for the $d$-formulation takes the form of
\[\rho_\theta: \mathbb{R}^d \times \mathcal{Y}\rightarrow \mathbb{R},\] which measures the predictive quality of the projection for the response $y$:
$\rho_\theta(U \transpose x, y)$ with some parameter $\theta$. To compute the gradient of $\rho_\theta$ with respect to $U$ on the Grassmannian manifold, we follow the program developed in \cite{edelman1998geometry}. First compute a partial derivative of $\rho_\theta$ with respect to $U$. Let the partial derivative of the $\rho_\theta$ function with respect to the first argument be denoted as
\[g_1\triangleq  \dot{\rho}_\theta(U\transpose x, y) \in \mathbb{R}^d.\]
By the chain rule, we have the partial derivative with respective to $U$ is given by
\[
\frac{d \rho_\theta}{d U} = xg_1\transpose \in \mathbb{R}^{D\times d}.
\]
Using equation (2.70) in \cite{edelman1998geometry}, we can calculate the gradient on the Grassmannian from this partial derivative
\[
\nabla \rho_\theta = (I - UU\transpose) \frac{d \rho_\theta}{d U} = (I - UU\transpose)x g_1\transpose.
\]
In many problems that we describe in Section \ref{sec:OSDR_models}, the gradient $g_1$ is one term or a sum a few terms. Hence, the gradient has the desired low-rank structure. 
%

\subsection{$D$-formulation}

The $D$-formulation assumes that the loss function is defined in the ambient dimension space:
\[
\varrho_\vartheta: \mathbb{R}^D\times \mathcal{Y}\rightarrow \mathbb{R}.
\]
This setting is particularly useful for denoising and imputing the missing data, where we will assume the signal $x$ lies near a low-dimensional subspace, and estimate a low-dimensional component $U\beta$ and use it to predict $y$. The loss function takes the form of $\varrho_\vartheta(U\beta, y)$. Following a similar derivation as above, the gradient on the Grassmannian can be written as
\[
\nabla \varrho_\vartheta = (I - UU\transpose) g_2 \beta\transpose,
\]
where the partial derivative of $\varrho_\vartheta$ with respect to the first arguement is given by
\[
g_2  \triangleq \dot{\varrho}_\vartheta(U\beta, y) \in \mathbb{R}^{D}.
\]
Again, $g_2$ is often only one term or a sum of a few terms, and hence $\nabla \varrho_\vartheta$ has the desired low-rank structure.

To estimate $\beta$, for the denoising setting, we may use
\begin{equation}
\beta = \arg\min_z \|x  - U z\| = U\transpose x.\label{beta_est_1}
\end{equation}
When there is missing data, we are not able to observe the complete data vector $x_t$, but are only able to observe a subset of entries $\Omega_t \subset \llbracket D\rrbracket$. Using an approach similar to that in \cite{GROUSE1}, $\beta$ is estimated as
\begin{equation}
\beta = \arg\min_z \|\Delta_{\Omega_t}(x - U z)\|, \label{beta_est_2}
\end{equation}
where $\Delta_{\Omega_t}$ is an $n\times n$ diagonal matrix which has 1 in the $j$th diagonal entry if $j \in \Omega_t$ and has 0 otherwise. It can be shown $\beta = (U_{\Omega}\transpose U_{\Omega})^{-1}U_{\Omega}\transpose x_\Omega$, where $U_\Omega = \Delta_\Omega U$ and $x_\Omega = \Delta_\Omega x$.%

\begin{algorithm}
\caption{OSDR: meta-algorithm}
\begin{algorithmic}[1]
\REQUIRE a sequence of predictors and responses $(x_t, y_t)$, initial model parameter $\theta$ and subspace $U \in \mathbb{R}^{D\times d}$.
\FOR {$t = 1, 2, \ldots$}
\STATE \COMMENT{$d$-formulation}\\
$\Delta \leftarrow  (I-UU\transpose)x \dot{\rho}_\theta(U\transpose x, y)$ \COMMENT{$\rho: \mathbb{R}^d\times \mathbb{Y}\rightarrow \mathbb{R}$}
\STATE  \COMMENT{$D$-formulation}\\
$\Delta \leftarrow (I-UU\transpose) \dot{\varrho}_\vartheta(U\beta, y)\beta$, \COMMENT{$\varrho: \mathbb{R}^D\times \mathbb{Y}\rightarrow \mathbb{R}$, \\$\beta$ estimated from $x$}
\STATE fix $\theta$ or $\vartheta$, update $U$ along Grassmannian gradient $\Delta$ in geodesic
\STATE fix $U$, update $\theta$ or $\vartheta$
\ENDFOR
\end{algorithmic}
\label{OSDR}
\end{algorithm}

\section{OSDR for specific models}\label{sec:OSDR_models}

In the section, we illustrate various forms of loss functions and show that the Grassmannian gradient with respect to $U$ typically takes the form of $\gamma r w\transpose$, for some scalar $\gamma \in \mathbb{R}$, vectors $r \in \mathbb{R}^{D}$, and $w \in \mathbb{R}^d$.

\subsection{Linear regression}

For linear regression, $y \in \mathbb{R}$ and the loss function will be the $\ell_2$-norm of prediction error. In the $d$-formulation, $\theta = (a, b)$ with $a \in \mathbb{R}^d$ and $b\in \mathbb{R}$, and the loss function is
\[
\rho_\theta(U\transpose x, y) \triangleq \|y - a\transpose U\transpose x - b\|.
\]
Define
\begin{equation}
\hat{y}=a\transpose U\transpose x + b.
\label{linear_y}
\end{equation}
The Grassmannian gradient of the loss function with respect to $U$ is given by
\[
\nabla \rho_\theta(U\transpose x, y) \triangleq -(y - \hat{y})\underbrace{(I-UU\transpose)x}_{r} \underbrace{a\transpose}_{w\transpose}
\]
Using the rank-one structure of the gradient above, we perform geodesic gradient descent for $U$ using a simple form. Write
\begin{equation*}
\begin{split}&\nabla \rho_\theta(U\transpose x, y) \\
\hspace{-0.3in}=&~ \begin{bmatrix} r/\|r\| & v_2& \ldots &v_d \end{bmatrix}
\mbox{diag}(\sigma)
\begin{bmatrix} w/\|w\| & z_2& \ldots &z_d \end{bmatrix}\transpose,
\end{split}
\end{equation*}
where
\begin{equation*}
\sigma =  -(y - \hat{y})\|r\|\|w\|, \label{sigma_up}
\end{equation*}
 $v_2, \ldots, v_d$ are an orthonormal set orthogonal to $r$ and $z_2, \ldots, z_d$ are an orthonormal set orthogonal to $w$. Subsequently, using the formula in \cite{edelman1998geometry} update of $U$ is given by
\begin{equation}
U_{\rm new} = U + \frac{(\cos(\sigma \eta)-1)}{\|w\|^2} U w w\transpose + \sin(\sigma \eta) \frac{r}{\|r\|} \frac{w\transpose}{\|w\|},
\label{U_update_1}
\end{equation}
where $\eta >0$ is a step-size.
Similarly, for a fixed $U$, we may find its gradient with respect to the regression coefficient vector and update via
\begin{equation}
\begin{split}
a_{\rm new} &= a+ \mu (y-\hat{y}) U\transpose x, \\
b_{\rm new} &= b + \mu (y-\hat{y}),
\end{split}
\label{aupdate}
\end{equation}
where $\mu > 0$ is step-size for the parameter update.

In the $D$-formualtion, the model parameters are $\vartheta \triangleq (c, e)$, with $c\in \mathbb{R}^D$ and $d\in \mathbb{R}$. 
Essentially, by replacing $x$ by its estimate $U\beta$, we have the loss function
\[
\varrho_\vartheta(U\beta, y) \triangleq \|y- c\transpose U\beta - e\|_2^2.
\]
where $\beta$ is estimated using the subspace from the previous step using (\ref{beta_est_1}) or (\ref{beta_est_2}).
Let $\hat{y} = c\transpose U \beta$, we can show
\[
\nabla \varrho_\vartheta(x, y) = -(y - \hat{y}) \underbrace{(I - UU\transpose) a}_{r}\underbrace{\beta\transpose}_{w\transpose},
\]
and hence the subspace can be updated similarly, and the model parameters are updated via
\begin{equation}
\begin{split}
c_{\rm new} &= c+ \mu (y-\hat{y}) U \beta,\\
e_{\rm new} &= e+ \mu (y-\hat{y}). \label{aupdate2}
\end{split}
\end{equation}

\begin{remark}[Difference from unsupervised tracking]
Due to an alternative formulation, update in OSDR differs from that in the unsupervised version \cite{GROUSE1, GROUSEproof}  in that the update in OSDR depends on $y - \hat{y}$. This is intuition since the amount of update we have on the subspace should be driven by the prediction accuracy for the response variable.
\end{remark}

\begin{remark}[Mini-batch update]
Instead of updating with every single new sample, we may also perform an update with a batch of samples. The Grassmannian gradient for this mini-batch update scheme can be derived as 
\[
\nabla \rho_\theta(x, y) \triangleq - \frac{1}{B}\sum_{i = t-B}^{t}(y_i - \hat{y}_i)(I-UU\transpose)x_i a\transpose.
\]
In this case the gradient is no longer rank-one. We may use a rank-one approximation of this gradient, or use the exact rank-$B$ update described in (\ref{subspace_update}), which requires computing an SVD of this gradient.
\end{remark}

\begin{remark}[Computational complexity]
The computational complexity of OSDR is quite low and it is $\mathcal{O}(Dd)$. The most expensive step is to compute $r = (I-UU\transpose)x$ or $r = (I-UU\transpose)a$, in the $d$- and $D$-formulations, respectively. This term can be computed as, for instance $x - U(U\transpose x)$, to have a lower complexity $\mathcal{O}(Dd)$ (otherwise the complexity is $\mathcal{O}(D^2d)$). 
\end{remark}

\subsection{Logistic regression}

For logistic regression, $y\in \{0, 1\}$. Define the sigmoid function \[h(x) = \frac{1}{1+e^{-x}}.\] The loss function for logistic regression corresponds to the negative log-likelihood function assuming $y$ is a Bernoulli random variable. For the $d$-formulation, the loss function is given by
\[
\rho_\theta(U\transpose x, y) = y\log h(a\transpose U\transpose x + b)
+ (1-y)\log (1-h(a\transpose U \transpose x +b)),
\]
and the model parameter $\theta = (a, b)$ with $a\in \mathbb{R}^{d}$ and $b \in \mathbb{R}$. For the $D$-formulation with a parameter estimate $\beta$, we have
\[
\varrho_\vartheta(U\beta, y) = y\log h(c\transpose U \beta  + e)
+ (1-y)\log (1-h(c\transpose U \beta + e))
\]
and the model parameter $\vartheta = (c, e)$ with $c\in \mathbb{R}^{D}$ and $e \in \mathbb{R}$. It can be shown that, in the logistic regression setting, the Grassmannian gradients with respect to $U$, for these two cases are given by
\[
\nabla \rho_\theta = -(y-\hat{y}) \underbrace{(I - UU\transpose)x}_{r} \underbrace{a \transpose}_{w\transpose},
\]
where the predicted response
\begin{equation}
\hat{y} \triangleq h(a\transpose U\transpose x +b),
\label{logistic_y}
\end{equation}
or
\[
\nabla \varrho_\vartheta = -(y - \hat{y})\underbrace{(I - UU\transpose)c}_{r} \underbrace{\beta\transpose}_{w\transpose},
\]
where the predicted response 
\begin{equation}
\hat{y} \triangleq h(c\transpose U \beta  +b).
\label{logistic_y}
\end{equation}
Note that the gradients for linear regression and logistic regression take a similar form, with the only difference being how the response is predicted: in linear regression it is defined linearly as in (\ref{linear_y}), and in logistic regression it is defined through the sigmoid function as in (\ref{logistic_y}). Hence, OSDR for linear regression and logistic regression take similar forms and only differs by what response function is used, as shown in Algorithm \ref{Alg:OSDR_linear}.

\begin{algorithm}
\renewcommand{\algorithmicrequire}{\textbf{Input:}}
\renewcommand\algorithmicensure {\textbf{Output:} }
    \caption{\label{Alg:OSDR_linear} OSDR for linear and logistic regressions}
    \label{alg:regression}
    \begin{algorithmic}[1]
       \REQUIRE $y_t$ and $x_t$ (missing data, given $x_t$ observed on an index set $\Omega_t$), $t = 1, 2, \ldots$. Initial values for $U$, $a$ and $b$ (or $c$ and $d$). Step-sizes $\eta$ and $\mu$. $\tilde{h}$ can be linear or sigmoid function. 
        \FOR {$t = 1, 2, \ldots$}
           \STATE  \COMMENT{$d$-formulation}\\
           $\hat{y} \leftarrow \tilde{h}(a\transpose U\transpose x + b)$,
           $r\leftarrow (I - UU\transpose) x$,\\
           $\sigma \leftarrow -(y - \hat{y})\|r\|\|a\|$
           \STATE  \COMMENT{$D$-formulation} \\
           $\beta \leftarrow (U_{\Omega}\transpose U_{\Omega})^{-1} U_{\Omega}\transpose x_{\Omega}$ \\
           $\hat{y}\leftarrow  \tilde{h}(c\transpose U \beta) + d$,
           $r\leftarrow (I-UU\transpose) c$,\\
           $\sigma \leftarrow -(y - \hat{y})\|r\|\|\beta\|$
            \STATE \COMMENT{update subspace}\\
           $U  \leftarrow U + \frac{\cos(\sigma \eta) - 1}{\|w\|^2}Uww\transpose + \sin(\sigma \eta)\frac{r\transpose}{\|r\|}\frac{w\transpose}{\|w\|}$

            \STATE \COMMENT{update regression coefficients and residuals} \\
           $a \leftarrow a + \mu (y-\hat{y}) U\beta $,
           $b \leftarrow b + \mu (y-\hat{y}) $
              \ENDFOR
    \end{algorithmic}
\end{algorithm}

\subsection{Multiple linear regression}

We may also extend OSDR to multiple linear regression, where $y \in \mathbb{R}^m$ for some integer $m$.
The loss functions for the $d$- and $D$-formulations are given by
\begin{align*}
\rho_\theta(U\transpose x, y) &= \|y - A\transpose U\transpose x \|_2^2, \\
\varrho_\vartheta(U \beta, y) &= \|y -C\transpose U\beta\|_2^2
\end{align*}
with $\theta = A \in \mathbb{R}^{d\times m}$, and $\vartheta = C \in \mathbb{R}^{D\times m}$. Here we assume the slope parameter is zero and this can be achieved by subtracting the means from the predictor vector and the response variable, respectively. 
It can be shown that
\[
\nabla \rho_\theta = -\underbrace{(I-UU\transpose)x}_{r} \underbrace{(y - \hat{y})\transpose A\transpose}_{w\transpose},\quad \hat{y} = A\transpose U\transpose x,
\]
and
\[
\nabla \varrho_\vartheta = - \underbrace{(I-UU\transpose)C (y - \hat{y})}_{r} \underbrace{\beta\transpose}_{w\transpose}, \quad \hat{y} = C\transpose U \beta.
\]
It can also be shown that the partial derivative of $\rho_\theta$ with respect to $A$ for a fixed $U$ is given by $-U\transpose x (y-\hat{y})\transpose$, and the partial derivative of $\varrho_\vartheta$ with respect to $C$ for a fixed $U$ is given by $-U\beta (y-\hat{y})\transpose$. OSDR for multiple linear regression is given in Algorithm \ref{alg:multi_linear}.

\begin{algorithm}
\renewcommand{\algorithmicrequire}{\textbf{Input:}}
\renewcommand\algorithmicensure {\textbf{Output:} }
    \caption{\label{Alg:OSDR} OSDR for multiple linear regression}
    \label{alg:multi_linear}
    \begin{algorithmic}[1]
         \REQUIRE $y_t$ and $x_t$ (missing data, given $x_t$ observed on an index set $\Omega_t$), $t = 1, 2, \ldots$. Initial values for $U$, $A$ and $b$ (or $C$ and $d$). Step-sizes $\eta$ and $\mu$.
        \FOR {$t = 1, 2, \ldots$}
           \STATE  \COMMENT{$d$-formulation}\\
           $\hat{y} \leftarrow A\transpose U\transpose x + b$,
           $r\leftarrow (I - UU\transpose) x$,\\
           $w\leftarrow (y - \hat{y})\transpose A\transpose$,
           $\sigma \leftarrow -\|r\|\|w\|$
           \STATE  \COMMENT{$D$-formulation} \\
           $\beta \leftarrow (U_{\Omega}\transpose U_{\Omega})^{-1} U_{\Omega}\transpose x_{\Omega}$ \\
           $\hat{y}\leftarrow  C\transpose U \beta + d$,
           $r\leftarrow (I-UU\transpose) C(y - \hat{y})$,\\ $w = \beta$,
           $\sigma \leftarrow -\|r\|\|w\|$
            \STATE \COMMENT{update subspace}\\
           $U  \leftarrow U + \frac{\cos(\sigma \eta) - 1}{\|w\|^2}Uww\transpose + \sin(\sigma \eta)\frac{r\transpose}{\|r\|}\frac{w\transpose}{\|w\|}$

            \STATE \COMMENT{update regression coefficients and residuals} \\
           $A \leftarrow A + \mu U\transpose x (y-\hat{y})\transpose$ \COMMENT{$d$-formulation} \\
           $C \leftarrow C + \mu U\beta(y-\hat{y})\transpose$ \COMMENT{$D$-formulation}
              \ENDFOR
    \end{algorithmic}
\end{algorithm}

\subsection{Multinomial logistic regression}

Multinomial logistic regression means that $y\in \{0, 1, \ldots, K-1\}$ for some integer $K$ is a categorical random variable and it is useful for classification. In the following, we focus on the $D$-formulation; the $d$-formulation can be derived similarly. The loss function is the negative likelihood function given by
\begin{equation}
\begin{split}
\rho_\theta(U\transpose x, y) &\triangleq -\sum_{k=0}^{K-2} \mathbb{I}\{y= j\} \log \left(\frac{e^{a_k\transpose U\transpose x + b_k}}{1
+\sum_{j=0}^{K-2} e^{a_k\transpose U\transpose x+b_k}}\right) \\
&\quad - \mathbb{I}\{y= K-1\} \log \left(\frac{1}{1+\sum_{j=0}^{K-2} e^{a_k\transpose U\transpose x+b_k}}\right).
\end{split}\nonumber
\end{equation}
In this case, the Grassmannian gradient will no longer be rank-one but rank-$K$, with
\[\nabla \rho_\theta = -(I-UU\transpose)\Sigma,\]
and
\begin{equation}
\begin{split}
\Sigma &= \sum_{k=0}^{K-2} \mathbb{I}\{y = k\}
           \left[a_k \beta\transpose - \frac{1}{1+e^{a_k\transpose U \beta + b_k}}
           \sum_{l = 0}^{K-1} e^{a_l\transpose U \beta + b_l} a_l \beta\transpose
           \right]\\
           &+ \mathbb{I}\{y = K-1\}\left[\frac{-e^{-a_{K-1}\transpose U \beta -b_{K-1}}}{1
          +\sum_{k=0}^{K-1} e^{a_k\transpose U \beta + b_k }}\right]\sum_{l = 0}^{K-1} e^{a_l\transpose U \beta + b_l} a_l \beta\transpose.
\end{split}
\label{big_Sigma}
\end{equation}
Note that $\Sigma$ consists of a sum of $K$ terms and, hence, is usually rank-$K$. We no longer have the simple expression to calculate update of $U$ along the geodesic and the precise update requires performing a (reduced) singular value decomposition of the gradient
\[\nabla \rho_\theta \triangleq P \Sigma Q\transpose,\]
where $\Sigma \in \mathbb{R}^{K\times K}$ is a diagonal matrix with the diagonal entries being the singular values. Using Theorem 2.3 in \cite{edelman1998geometry}, we update $U$ as
\begin{equation}
U = \begin{bmatrix} UQ & P\end{bmatrix}
\begin{bmatrix}
\cos(\Sigma \eta)\\
\sin(\Sigma \eta)
\end{bmatrix} Q\transpose,\label{subspace_update}
\end{equation}
where $\eta > 0$ is the step-size. Alternatively, we may use the rank-one approximation to the Grassmannian gradient to derive, again, a simple update, which is given by Algorithm \ref{alg:multinomial}.

\begin{algorithm}
\renewcommand{\algorithmicrequire}{\textbf{Input:}}
\renewcommand\algorithmicensure {\textbf{Output:} }
    \caption{OSDR for $K$-multinomial logistic regression}
    \label{alg:multinomial}
    \begin{algorithmic}[1]
       \REQUIRE $y_t$ and $x_t$ (missing data, given $x_t$ observed on an index set $\Omega_t$), $t = 1, 2, \ldots$. Initial values for $U$, $a_k$. Step-sizes $\eta$ and $\mu$.
        \FOR {$t = 1, 2, \ldots$}
         \STATE  \COMMENT{$D$-formulation} \\
           $\beta \leftarrow (U_{\Omega}\transpose U_{\Omega})^{-1} U_{\Omega}\transpose x_{\Omega}$ \\
                      \STATE \COMMENT{predict response} \\
           $\hat{y}_k \propto e^{a_k\transpose U \beta  + b_k}$, $k = 1, \ldots, K-1$\\
           $\hat{y}_k \leftarrow \hat{y}_k/(1+\sum_{l=1}^{K-1}\hat{y}_l)$\\
           $\hat{y}_{K} \leftarrow 1 - \sum_{l=1}^{K-1}\hat{y}_l$, \\
          \STATE compute $\Sigma$ using (\ref{big_Sigma})
           \STATE find the dominant singular value $\sigma$, corresponding left singular vector $\rho$ and right singular vector $r$ for $(I-UU\transpose)\Sigma$
           \STATE \COMMENT{update subspace}\\
           $p \leftarrow U r$\\
           $U \leftarrow U + \frac{\cos(\sigma \eta) - 1}{\|r\|} U rr\transpose
           + \sin(\sigma \eta)\frac{\rho}{\|\rho\|}\frac{r\transpose}{\|r\|}
           $
           \STATE \COMMENT{Update regression coefficients} \\
          $h = \mathbb{I}\{y=k\}-\frac{\sum_{l=1}^{K-1}\mathbb{I}\{y = l\} e^{a_l\transpose U \beta + b_l}}{1+\sum_{l=1}^{K-1}e^{a_k\transpose U \beta + b_k}}$\\
           $\qquad- \mathbb{I}\{y = K\}\frac{e^{a_k\transpose U \beta + b_k}}{1+\sum_{l=1}^{K-1}e^{a_k\transpose U \beta + b_k}}$\\
           $a_{k} \leftarrow a_{k} + \mu h U \beta, \quad k= 1, \ldots, K-1 $\\
           $b_{k} \leftarrow b_{k} + \mu h, \quad k = 1, \ldots, K-1$
        \ENDFOR
    \end{algorithmic}
\end{algorithm}

\subsection{Support vector machine (SVM)}

The loss function for SVM is the hinge loss. For the $d$- and $D$-formulations, the loss functions are
\[
\rho_\theta(U\transpose x, y) = \max\{0, 1-y a\transpose U\transpose x\},
\]
where $\theta = a\in \mathbb{R}^d$, and
\[
\varrho_\vartheta(U \beta, y) = \max\{0, 1- y c\transpose U \beta\},
\]
where $\theta = c\in \mathbb{R}^D$. Note that the loss function is not differentiable. We may use its sub-gradient to perform gradient descent, or find a smooth surrogate function to approximate the hinge loss.
The Grassmannian sub-gradients for the two loss functions are 
\[
\nabla \rho_\theta = \frac{y}{2}[\mbox{sgn}(ya\transpose U\transpose x + 1) +1]\underbrace{(I-UU\transpose)x}_{r} \underbrace{a\transpose}_{w\transpose},
\]
and
\[
\nabla \varrho_\vartheta = \frac{y}{2}[\mbox{sgn}(yc\transpose U \beta + 1) +1]\underbrace{(I-UU\transpose)c}_{r} \underbrace{\beta\transpose}_{w\transpose}.
\]
These gradients are again rank-one and, hence, we may update $U$ along geodesic efficiently.

\subsection{Random dot product graph model} 

The random dot product graph model is useful for relational data which usually occurs in social network study \cite{NickelThesis2006,YoungScheinerman2007,randomDot2}.
The model assumes that each node is associated with a feature $\beta_i$, and an edge between two nodes are formed with a probability proportional to the inner product between their features $\beta_j\transpose \beta_j$. 
Suppose at each time $t$, we observe a pair of nodes in the network with predictor vectors $x_{1, t}$ and $x_{2, t}$ as well as their relation indicator variable $y_t \in \{0, 1\}$ (i.e., an edge is formed or not). We assume a logistic regression model $\mathbb{P}(y_t = 1) = h(a_t\beta_{1, t}\transpose \beta_{2, t} + b_t)$ for some feature vectors $\beta_{1, t}$ and $\beta_{2, t}$ that are projections of $x_{1, t}$ and $x_{2, t}$. Here our goal is to choose a proper subspace that fits the data nicely. 
 
Note that given $x_{1}$ and $x_{2}$, the inner product can be estimated as $x_{1} \transpose UU\transpose x_{2}$, which involves a quadratic term in $U$ (rather than linear in $U$ as in other cases). To be able to obtain the rank-one structure, we use a two-step strategy: first fix $\beta_{2} = U\transpose x_{2}$ and update $U$, and then fix $\beta_{1} = U\transpose x_{1}$ and update $U$. The log-likelihood function for fixed $\beta_{2} = U\transpose x_{2}$ is given by
\begin{align*}
\varrho_\vartheta(U\beta_2, y) &= y \log h(a  x_1\transpose U \beta_2 + b)\\
&~~~~~+ (1-y) \log (1-h(a  x_1 \transpose U \beta_2 + b)).
\end{align*}
Similar to logistic regression,
\begin{equation*}
\nabla \varrho_\vartheta = (y - h(a x_1\transpose U \beta_2+b))\underbrace{(I-UU\transpose) ax_1}_{r} \beta_2\transpose,
\label{grouse_grad}
\end{equation*}
which is rank-one and we may update the subspace similarly. 
In the second step, the log-likelihood function for fixed $\beta_1 = U\transpose x_1$ is given by
\begin{align*}
\varrho_\vartheta(U\beta_1, y) &= y \log h(a \beta_1\transpose U\transpose x_2 + b) \\
& ~~~~~~~ + (1-y) \log (1-h(a \beta_1\transpose U\transpose x_2 + b)),
\end{align*}
and the subspace $U$ can be updated similarly. Finally, we fix $U$ (and hence fix $\beta_1 = U\transpose x_1$ and $\beta_2 = U\transpose x_2$) and update the logistic regression parameters as
\[
a_{\rm new} = a + \mu(y - h(a \beta_1\transpose \beta_2 + b))\beta_1\transpose \beta_2,
\]
\[
b_{\rm new} = b + \mu(y - h(a \beta_1\transpose \beta_2 + b)).
\]
Description of the complete algorithm is omitted here as it is similar to the case of logistic regression.

\subsection{Union-of-subspaces model} 

Union-of-subspaces model \cite{high_rank_MC12} and multi-scale union-of-subspace model \cite{SSP12,AllardChenMaggioni2011,MOUSSE2013} have been used to approximate manifold structure of the state. As illustrated in Fig. \ref{tree}, the multi-scale union-of-subspace is a tree of subsets defined on low-dimensional affine spaces
$
\bigcup_{(j, k) \in\mathcal{A}_t}
\mathcal{S}_{j, k, t}, $
with each of these subsets lying on a low-dimensional hyperplane with
dimension $d$ and is parameterized by
\begin{equation*}
  \begin{split}
    \mathcal{S}_{j, k, t}
     = \{\beta \in \mathbb{R}^d: & \quad v = U_{j, k, t} \beta + c_{j, k, t},\\
    & \quad \beta \transpose \Lambda_{j, k, t}^{-1} z \leq
    1, \quad \beta \in \mathbb{R}^d\},
  \end{split} 
\end{equation*}
where $j \in \{1, \ldots, J_t\}$ denotes the scale or level of the subset in the tree, $J_t$ is the tree depth at time $t$, and $k \in \{1, \ldots, 2^j\}$ denotes the index of the subset for that level. 
 The matrix $ U_{j, k, t} \in\mathbb{R}^{D\times d} $ is the
subspace basis, and $c_{j, k, t} \in \mathbb{R}^{D}$ is the offset of
the subset from the origin. The diagonal matrix \[{\Lambda}_{j, k,
  t} \triangleq \mbox{diag}\{\lambda_{j, k, t}^{(1)}, \ldots,
\lambda_{j, k, t}^{(d)}\} \in \mathbb{R}^{d\times d},\] with
$\lambda_{j, k, t}^{(1)}\geq \ldots \geq \lambda_{j, k, t}^{(d)}\geq
0$, contains eigenvalues of the covariance matrix of the projected
data onto each hyperplane. This parameter specifies the shape of the ellipsoid by capturing the spread of the data within the hyperplanes. 

We may couple the subspace tree with regression model, by attaching a set of regression coefficients $\{a_{j, k, t}, b_{j, k, t}\}$ with each subset. Given $x$, we may find a subset in the union that has the smallest affinity, and use that subset to estimate $\beta$ by projection onto the associated subspace and use the associated (linear or logistic) regression coefficients to predict $y$. The affinity can be a distance to the subset similar to what has been used for discriminate analysis or in \cite{MOUSSE2013},
\[
(j^*, k^*)= \arg\min_{j, k} \min_w (x - U_{j, k, t} w)^\top \Lambda_{j, k, t} (x - U_{j, k, t} w),
\]
or simply the distance to a subspace
\[
(j^*, k^*)= \arg\min_{j, k} \min_w \|x - U_{j, k, t} w\|.
\]
Then we predict the local coefficient associated with that subset.
OSDR can be derived for these models by combining a step of finding the subset with the smallest affinity with subsequent subspace and parameter update similar to the linear or logistic regression OSDR. 

We may also use this model together with the random dot product graph model for social networks, where two users may belong to two different subsets in the tree and their interaction is determined by the regression model associated with their common parent in the tree. This may capture the notion of community: each node represents one community and there is a logistic regression model for each community. The probability that two users interact is determined by the ``smallest'' community that they are both in. In this case, OSDR will be two-stage: classification based on affinity function followed by a two-step subspace update similar to the OSDR for the random dot product model. Section \ref{eg:union} presents one such example for illustration. 

\begin{figure}
\begin{center}
\includegraphics[width = .45\linewidth]{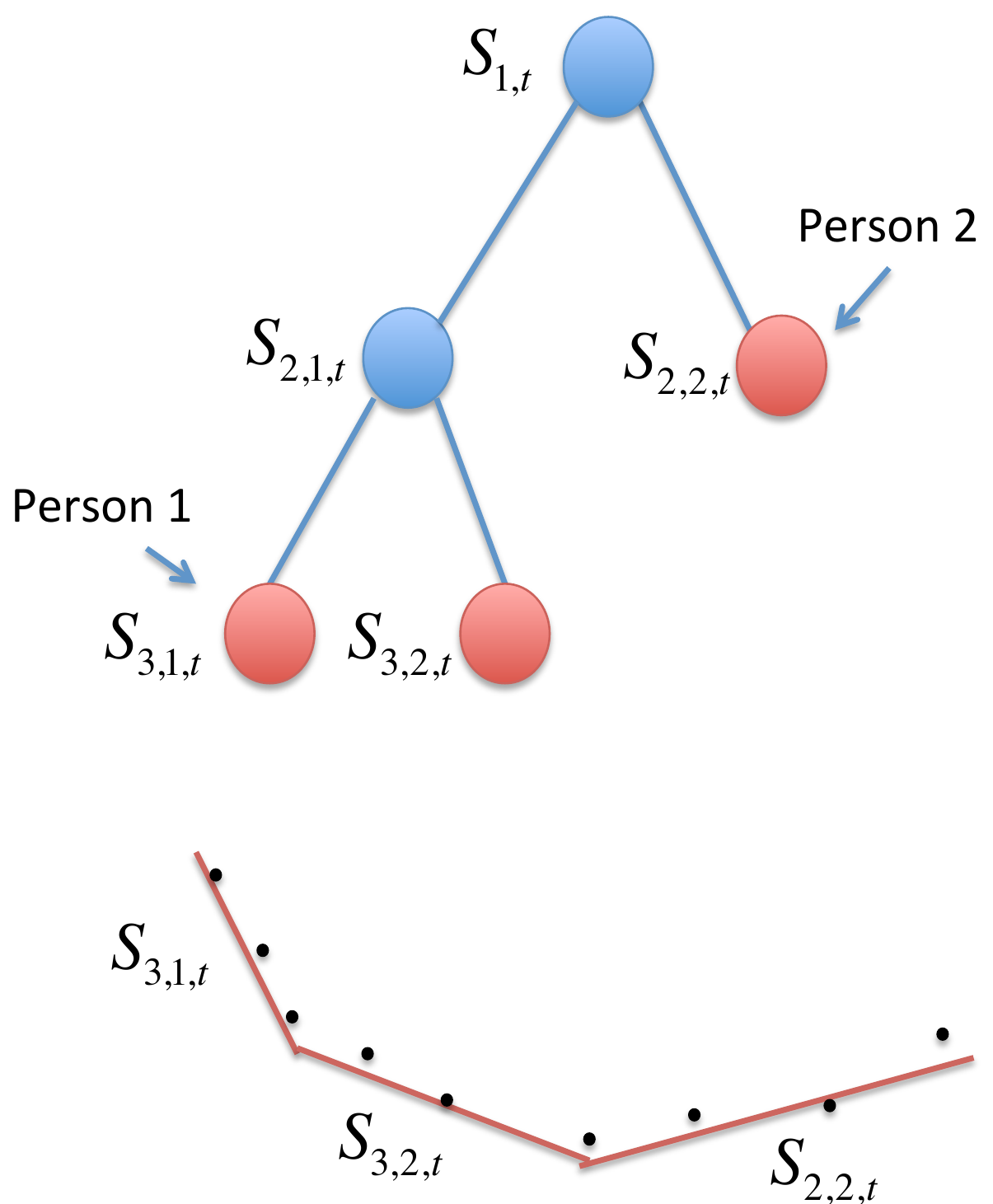}
\vspace{-0.1in}
\end{center}
\caption{Multi-scale union-of-subspaces model \cite{AllardChenMaggioni2011}. Data are bi-partitioned iteratively to form nested subsets, and a low-dimensional affine space (with offset from the origin) is fitted for data in each subset. The bi-partitioning creates a binary tree with multi-scale representation of the data: nodes of the tree contain parameters for that subset, and  leaves represent the union-of-subspaces used at time $t$. The lower panel illustrates the structure of these subspaces. The black dots represent historical data $x_t$ used to estimate the subspace. This model can be used in conjunction with, for example, a random dot product graph to model social networks: at each time $t$, two persons with features $\beta_1$ and $\beta_2$ from two subsets interact with probability proportional to the inner product $\beta_1\transpose \beta_2$.
  }
  \vspace{-0.1in}
\label{tree}
\end{figure}

\section{Theoretical analysis}\label{sec:theory}

The general OSDR problem is hard to analyze due to its non-convexity and bilinearlity in $U$ and the model parameters. In this section, we study the optimization problem for linear regression with the $D$-formulation to obtain some theoretical insights. 
In the linear regression case, the loss function is the $\ell_2$ norm
$
\varrho_\vartheta (U \beta, y)=\mathbb (y-c^\intercal U \beta)^2.
$
When there is no missing data, the projection coefficient is given by $\beta= U^\intercal x$. In a simplified setting, assume the response variable is generated with the parameter $c$: $
y = c^\intercal x.$ Then the loss function is given by
\[
\varrho_\vartheta(U \beta, y)  =(c^\intercal (I-UU^\intercal )x)^2. \]

\subsection{Fixed-point with respect to $U$}

First, we show that for a fixed model parameter, the optimization problem with respect to the subspace $U$ will converge to an orthonormal matrix even without the orthogonality constraint. We make a simplifying assumption that the true response is linear with parameter equal to the assumed parameter $c$: $y = c\transpose x$, then we have that one step in the alternating minimization can be written as
%
%
\begin{equation}\label{opt:b}
\begin{array}{rl}
\underset{U}{\mbox{minimize}}\ & \mathbb (c\transpose (I-UU^\intercal)x)^2\\
{\rm subject\ to\ }& U\transpose U = I_d.
\end{array}
\end{equation}
This problem is non-convex due to the constraints as well as the quadratic term $UU^{\intercal}$ in the objective function.
Without loss of generality, we may assume $\Vert c \Vert_2 =1$. 
Construct a matrix $C_0$ with $c$ being its first column.
Then we consider the following optimization problem related to (\ref{opt:b}) that will help us to establish properties later.
\begin{equation}\label{opt:block}
\begin{array}{rl}
\underset{U}{\rm minimize} & L(U, C_0)\triangleq \mathbb E\Vert C_0^\intercal (I-UU^\intercal)x\Vert_2^2\\
{\rm subject\ to\ } & U^\intercal U = I_d.
\end{array}
\end{equation}

\begin{theorem}\label{thm:stationary}
Given a fixed orthogonal matrix $C_0$, the stationary point $U^*$ to the optimization problem (\ref{opt:block}) without the constraint:
\begin{equation}\label{opt:blockcnvx}
\underset{U}{\rm minimize}~ \mathbb E \Vert C_0^\intercal (I-UU^\intercal)x\Vert_2^2,
\end{equation}
are orthogonal matrices of size $D\times d$ whose columns are $d$ largest eigenvectors of the covariance matrix $X$ of data $x$. Assume $X$ has $d$ distinct dominant eigenvalues. 
\end{theorem}
We need the following lemma for the proof.
\begin{lemma}[\cite{yang1995projection}]\label{lemma:stationary1}
Let $C_0\in \mathbb R^{D\times D}$ be a positive semi-definite matrix and $U\in \mathbb R^{D\times d}$.
For a function $J: \mathbb R^{D\times d}\mapsto \mathbb{R}$ defined as
\[
J(U) = {\rm tr} (C_0) - 2{\rm tr} (U^\intercal C_0 U) + {\rm tr} (U^\intercal C_0 U\cdot U^\intercal U),
\]
the gradient of $J$ is
\[
\nabla J(U) = -2[2C_0-C_0UU^\intercal-UU^\intercal C_0] U.
\]
$U^{*}$ is a stationary point of $J(U)$ if and only if $U^{*} = U_d Q$, where $U_d\in \mathbb R^{D\times d}$ contains any $d$ distinct eigenvectors of $C$ and $Q\in R^{d\times d}$ is an arbitrary orthonormal matrix. Moreover, all stationary points of $J(U)$ are saddle points except when $U_d$ contains the $d$ dominant eigenvectors of $C_0$, in which case $J(U)$ attains the global minimum at $U^{*}$.
\end{lemma}
\begin{proof}[Proof of Theorem \ref{thm:stationary}]
Let $X$ denote the covariance matrix of $x$: $X = \mathbb{E}[xx\transpose]$, and let $\tilde{C} = C_0C_0\transpose$.
We may write the objective function as
\begin{align*}
L(U, C_0)
&= \mathbb{E}[{\rm tr}(C_0^\intercal (I-UU^\intercal)xx^\intercal (I-UU^\intercal) C_0)]\\
&= {\rm tr}( (X-XUU^\intercal -UU^\intercal X + UU^\intercal XUU^\intercal) \tilde{C})\\
&=-2[\tilde{C}(I-UU^\intercal)X + X(I-UU^\intercal \tilde{C})]U.
\end{align*}
Then the partial derivative of $L(U, C_0)$ with respect to $U$ is then given by
\begin{align*}
\frac{d L(U, C_0)}{d U}
&= - \frac{d ({\rm tr}(XUU^\intercal \tilde{C}))}{dU} - \frac{d({\rm tr}
(UU^\intercal X \tilde{C}))}{dU} \\
&~~~+ \frac{({\rm tr}(UU^\intercal XUU^\intercal \tilde{C}))}{dU}\\
&= -2(\tilde{C}UU^\intercal X + XUU^\intercal \tilde{C} - \tilde{C}X-X\tilde{C})U.
\end{align*}
If we choose the columns of $C_0$ properly that $C_0=(c, c_2,\ldots, c_D)$ is orthonormal, we have $\tilde{C}=C_0C_0\transpose=I_D$, and thus,
 \[
\frac{d L(U,C_0)}{dU}= -2[(I-UU^\intercal)X + X (I-UU^\intercal )]U.
\]
With the equation above together with Lemma \ref{lemma:stationary1}, we have that the only stationary points of the optimization problem $U^{*}$ are $d$ distinct dominant eigenvectors of the matrix $X$ (assuming $X$ is full-rank).

\end{proof}

\subsection{Convergence}

In the same setting, if we fix the model parameter, we may establish the following local convergence property with respect to the Grassmannian gradient of $U$. Suppose the case where $x$ is exactly on the subspace ${U^*}$, and $x=U^* s$.
We use $\phi_i(U_t, U^*)$, the principal angles between $U_t$ and the true subspace ${U^*}$, which is defined as
\[
\cos \phi_i(U_t, U^*) = \sigma_i({U^*}\transpose U_t), \ i\in \llbracket d\rrbracket
\]
as a metric.
Further define
\[
\epsilon_t \triangleq \sum_{i=1}^{d} \sin^2 \phi_i (U^*, U_t) = d-\| {U^*}^\intercal U_t \|_F^2.
\]
Note that when there is no missing data, $p=U\beta$, $r=(I-UU\transpose)a$, and
\[
r\transpose p = 0.
\]
Typically we can assume  $y_t-\hat{y}_t\neq 0$. Hence, we can choose a set of step-sizes $\mu_t>0$ properly such that
\begin{equation}\label{rightangle}
\Vert r \Vert^2 = \Vert x \Vert^2 - \| p \|^2 = \| s \|^2 - \|\beta\|^2.
\end{equation}
Define $\theta_t$ such that
\[
\cos \theta_t  = \frac{\| p_t \|}{\| x_t \|},
\quad
\sin \theta_t = \frac{\| r_t \|}{\|x_t\|}.
\]
If such $\mu_t>0$ exists, we may choose the constant $c_t = 1$; otherwise we may choose $c_t$ accordingly to satisfy (\ref{rightangle}).
When there is no missing data, it is easy to check that Lemma 3.1 in \cite{GROUSEproof} still applies:
if we choose the step size $\eta_t$ such that
$
\eta_t = \theta_t/\sigma_t,
$
then we have
\begin{equation}
\epsilon_t - \epsilon_{t+1} = (1-\frac{\beta_t^\intercal A_tA_t\beta_t}{\beta_t^\intercal \beta_t}),
\end{equation}
where
$
A_t = U_t\transpose {U^*}.
$

Next, we establish conditions for the alternating minimization used in OSDR to converge. The following lemma from \cite{niesen2009adaptive} comes handy.
\begin{lemma}[Theorem 4 in \cite{niesen2009adaptive}]\label{lemma:cited}
Let $(\mathcal M, d)$ be a compact metric space. Given two sets $\mathcal P,\mathcal Q\subset \mathcal M$, define the Hausdorff distance between them as 
\begin{align*}
d_{H} (\mathcal P, \mathcal Q) \triangleq \max\Big\{ &{\rm sup}_{P \in \mathcal P} {\rm inf}_{Q\in \mathcal Q} d(P, Q) ,\\& {\rm sup}_{Q \in \mathcal Q} {\rm inf}_{P\in \mathcal P} d( P,  Q)\Big\}.
\end{align*}
Let $\{(\mathcal P_n, \mathcal Q_n)\}_{n\geq 0}$, $\mathcal P$, $\mathcal Q$ be compact subsets of the compact metric space $(\mathcal M, d)$ such that
\[
\mathcal P_n \xrightarrow{d_H}\mathcal P, \mathcal Q_n \xrightarrow{d_H} \mathcal Q
\]
 and let $\ell: \mathcal M\times \mathcal M\rightarrow \mathbb R$ be a continuous function. If there exists a function $\delta: \mathcal M\times \mathcal M\rightarrow \mathbb R$, and the following two conditions hold: \\ 
(a) for all $n\geq 1$, $P\in \mathcal P_n$, $\tilde{Q}\in \mathcal Q_{n-1}$, $\tilde{P} = \arg \min_{P\in\mathcal P_n} \ell(P,\tilde{Q})$
\[
\delta(P,\tilde{P}) +  \ell(\tilde{P},\tilde{Q})\leq \ell(P,\tilde{Q}),
\]
(b) for all $n\geq 1$, $P,\tilde{P}\in \mathcal P_n$, ${Q}\in \mathcal Q_n$, $\tilde{Q} = \arg \min_{Q\in\mathcal Q_n} \ell(\tilde{P},{Q})$,
\[
\ell(P,\tilde{Q})\leq \ell(P,Q) +\delta(P,\tilde{P}),
\]
then the adaptive alternating minimization algorithm 
\[
P_n^* \in \arg \min_{P\in \mathcal P_n} \ell (P, Q_{n-1}),
\]
\[
Q_n^* \in \arg \min_{Q\in \mathcal Q_n} \ell (P_n, Q),
\]
with $(P^*_0, Q^*_0)$ as an initialization converges.
\end{lemma}
Lemma \ref{lemma:cited} can be applied to our setting in the following sense. 
Let
$
q = (I- UU^\intercal)x
$
and
$
P = cc^\intercal,
$
we have that our optimization problem
\begin{equation}
\begin{array}{rl}
\underset{U, c}{\rm minimize}\ & (c^\intercal (I-UU^\intercal )x)^2\\
{\rm subject\ to\ } & U^\intercal U = I_d\\
& y_t = c^\intercal x_t,
\end{array}
\end{equation}
can be recast into
\begin{equation}
\begin{array}{rl}
\underset{U, c}{\rm minimize}\ & \ell(P,q) \triangleq q^\intercal P q\\
{\rm subject\ to\ } & P\in\mathcal P_t,\\
& q\in \mathcal Q_t,
\end{array}
\end{equation}
where
\[
\mathcal Q_t = \{ z\in \mathbb R^D: \exists V\in \mathbb R^{D\times d}, V^\intercal V = I_d, z = (I-VV^\intercal)x_t\},
\]
and 
\[
\mathcal P_t = \{Z\in \mathbb R^{D\times D}: \exists v\in \mathbb R^D, y_t = v^\intercal x_t, Z = vv^\intercal \}.
\]
For simplicity, denote $P_n, \tilde{P}_n$ and $q_n, \tilde{q}_n$ be arbitrary items from $\mathcal P_n$ and $\mathcal Q_n$, respectively, and 
\[
P_{q, n} \triangleq \arg \min_{P\in \mathcal P_n} \ell (P, q),
\]
and 
\[
q_{P, n} \triangleq \arg \min_{q\in {\mathcal Q}_n} \ell(P, q).
\]
Suppose we can choose the data such that $\lim_{t\to \infty} x_t = \mu$ and $y_t$ should also converge to some $\nu$, then there exists $\mathcal P$, $\mathcal Q$ such that 
$\mathcal P_t\xrightarrow{d_H} \mathcal P$ and $\mathcal Q_t \xrightarrow{d_H} \mathcal Q$. 
In order to be able to choose $\delta(\cdot,\cdot )$ such that
\[
\delta(P_n, \tilde{P}_n)\geq \ell(P_n, q_{\{\tilde{P}_n, n\}}) - \ell (P_n, q_n),
\]
\[
\delta(P_n, P_{\{\tilde{q}_{n-1}, n\} }) \leq \ell(P_n, \tilde{q}_{n-1})-\ell(P_{\{\tilde{q}_{n-1}, n\}}, \tilde{q}_{n-1}), 
\]
we will need to have
\begin{align*}
&\ell(P_n, \tilde{q}_{n-1})-\ell(P_{\{\tilde{q}_{n-1}, n\}}, \tilde{q}_{n-1}) \\
\geq  
&\ell(P_n, q_{\{P_{\{\tilde{q}_{n-1}, n\} }, n\}}) - \ell (P_n, q_n),
\end{align*}
then we will need
\begin{equation}\label{delta-inequality}
\begin{array}{rl}
 &\ell(P_n, \tilde{q}_{n-1})-\ell(P_{\{\tilde{q}_{n-1}, n\}}, \tilde{q}_{n-1}) \\
\geq  
&\ell(P_n, q_{\{P_{\{\tilde{q}_{n-1}, n\} }, n\}}) - \ell (P_n, q_{\{P_n, n\}}).
\end{array}
\end{equation}
If the input sequence $\{x_t\}$ is properly selected such that for any $P_n\in \mathcal P_n$ and $q_{n-1}\in \mathcal Q_{n-1}$, the inequality (\ref{delta-inequality}) holds, then with Lemma \ref{lemma:cited}, we have the adaptive alternating algorithm for the linear regression problem converges.

\subsection{$\varepsilon$-net}

Another technique we may explore to tackle the non-convex optimization problem in (\ref{opt:blockcnvx}) is the efficient discretization. The idea is that instead of using alternating minimization, for the non-convex optimization problem involved, we can find a sufficiently fine yet efficient discretization (as function of the
desired error guarantee) that allows us to replace a single non-convex
optimization problem by a polynomial number of convex problems.
This will not lead to practically efficient algorithms as
everything beyond quadratic or cubic running time in the size is
usually prohibitive. However, it will allow us to
 establish general, theoretical guarantees and guide the
 search for better practical algorithms.
In particular, we can adopt approaches similar to the
\(\varepsilon\)-net approach in \cite{epsilon_net_2014}. This provides a discrete set $S$ of size
$|S| \leq (1/\varepsilon)^d$ so that for all $y \in \mathbb R^n$, there exists
a point $y_0 \in S$ such that $\|Ay - y_0\|_\infty <
\varepsilon$. This approximation now allows us to handle bilinear terms
of the form \(x\transpose A y\), which are non-convex, by replacing
them with the two terms \(x\transpose y_0\) and $\|Ay - y_0\|_\infty <
\varepsilon$ and iterating over all possible choices \(y_0 \in S\).
%
%
With the help of the following lemma, we may establish our result.
\begin{lemma}[$\varepsilon$-net for positive semidefinite matrix \cite{epsilon_net_2014}]\label{lemma:epsilon-net}
Let $A=BB^\intercal$, where $A$ is an $D\times D$ positive semidefinite matrix with entries in $[-1, 1]$ and $B$ is $D\times d$. Let $\Delta = \Delta_n =\{ x\in \mathbb R^D, \Vert x\Vert_1=1, x\geq 0 \}.$ There is a finite set $\mathcal S\in \mathbb R^d$ independent of $A, B$ such that
\[
\forall x\in \Delta, \exists \tilde{x} \in \mathcal S {\ \rm s.t.\ } \Vert B^\intercal x-\tilde{x} \Vert_{\infty}\leq \frac{\varepsilon}{d}
\]
with $\vert \mathcal S\vert=\mathcal O((1/\epsilon)^d)$. Moreover, $\mathcal S$ can be computed in time $\mathcal O ((1/\varepsilon)^d {\rm poly} (D)$).
\end{lemma}
Assume $\Vert x \Vert_1=\alpha>0$. From lemma \ref{lemma:epsilon-net}, we can compute such a set $\mathcal S$ in time $\mathcal O ((\alpha/\varepsilon)^d {\rm poly} (D))$ such that $\exists \tilde{x}\in \mathcal S$ and $\Vert U^\intercal x-\tilde{x}\Vert_\infty \leq {\varepsilon}/{d}$.
Let $\mathcal S=\{ \tilde{x}_1,\ldots. \tilde{x}_{\vert \mathcal S \vert} \}$ where $\vert \mathcal S \vert=\mathcal O((\alpha/\varepsilon)^d)$. 
We can then approximate a related problem to (\ref{opt:block}) by a family of $\vert\mathcal  S\vert$ problems:
\begin{equation}\label{opt:epsilon-net}
\min_{U}\  \mathbb E\Vert C_0^\intercal (x - U\tilde{x}_i )\Vert_2^2, \quad i=1,\ldots, \vert\mathcal S\vert.
\end{equation}
By combining all sub problems in (\ref{opt:epsilon-net}) together, we have the following equivalent problem
\begin{equation}\label{opt:epsilon-net-minimax}
\min_{U}\ \mathbb E\Vert C_0^\intercal (X_{\mathcal S} - U\tilde{X}_{\mathcal S} )\Vert_F^2,
\end{equation}
where \[X_{\mathcal S} = \underbrace{\begin{bmatrix}x & \cdots & x\end{bmatrix}}_{\vert \mathcal S\vert},\] and $\tilde{X}_{\mathcal S} = \begin{bmatrix}\tilde{x}_1 & \ldots & \tilde{x}_{\vert \mathcal S \vert }\end{bmatrix}$.
Note that the set $\mathcal S$ depends on $x$, and by construction $C_0$ is orthogonal, 
we approximate (\ref{opt:epsilon-net-minimax}) by
\begin{equation}\label{opt:epsilon-net-approx}
\min_{U} \Vert C_0^\intercal (X_{\mathcal S} - U\tilde{X}_{\mathcal S} )\Vert_F^2 = \min_U \Vert X_{\mathcal S} - U\tilde{X}_{\mathcal S} \Vert_F^2,
\end{equation}
which is convex.
\begin{lemma}
An $\varepsilon$-net approximation to (\ref{opt:blockcnvx}) can be computed in polynomial time
$\mathcal O ((\alpha/\varepsilon)^{3d}\cdot {\rm poly} (Dd)$.
\end{lemma}
\begin{proof}
The complexity of this epsilon-net method is the time for computing $\mathcal S$ plus the time for computing the optimization group \ref{opt:epsilon-net} generated by epsilon-net. From lemma \ref{lemma:epsilon-net}, computing $\mathcal S$ takes time  $\mathcal O ((\alpha/\varepsilon)^d {\rm poly} (D))$ and the size of the set is $\vert\mathcal S\vert=\mathcal O((\alpha/\varepsilon)^d)$. The dimension of the matrix in this problem is $D\times \vert \mathcal S\vert$. With the result in \cite{PlanThesis2011}, we may use an iterative iterative algorithm to solve the problem and the computational complexity in each iteration is $\mathcal O (D^2\vert \mathcal S\vert +\vert\mathcal S\vert ^3)$. This leads to a polynomial time algorithm with complexity $\mathcal O({\rm poly} (Dd))$.  Hence, the total time needed for generating and solving the $\varepsilon$-net approximation problems is
\begin{align*}
&\mathcal O ((\alpha/\varepsilon)^d {\rm poly} (D)+ \mathcal O((\alpha/\varepsilon)^{3d}) \cdot \mathcal O({\rm poly}(Dd))\\
=& \mathcal O ((\alpha/\varepsilon)^{3d}\cdot {\rm poly} (Dd).
\end{align*}
\end{proof}

\section{Numerical examples}\label{sec:num_eg}

We demonstrate the performance of OSDR compared with the online dimensionality reduction (ODR) (learning a subspace online by minimizing $\|x - U\beta\|$ without using information of $y$) via numerical examples on simulated data and real data. We start with a simple numerical example, followed by a larger scale simulation, an example to illustrate the union-of-subspaces and random dot product model, and finally real-data experiments. 

\subsection{Simple subspace tracking}

\noindent{\it Static subspace, logistic regression.} We first generate data by embedding a \emph{static} low-dimensional space of  intrinsic dimension $d^* = 2$ into a high dimensional space with dimension $D$, and generating a sequence of $\beta_t$ vector such that the entries of $\beta_t$ are i.i.d. $\mathcal{N}(0, 1)$ and lies within an ellipse:
\begin{equation}
\beta_{t,1}^2/r_1^2 + \beta_{t,2}^2/r_2^2 \leq 1. \label{ellipse}
\end{equation}
The predictor vector $x_t \in \mathbb{R}^p$ is formed as
$x_t = U\beta_t + w_t,$ where $w_t$ is a Gaussian noise with zero mean and variance equal to $10^{-3}$. 
The logistic regression coefficient vector is along the shorter axis of the ellipse. Among the 6000 samples generated this way,  the first 3000 samples are for training, and the remaining 3000 samples are for testing. Given $x_t$, ODR or OSDR predict the label $\hat{y}_t$, then we reveal the true label to calculate error, and then use $(x_t, y_t)$ to update. We use misclassification error $P_e$ on the test data as our performance metric. 
In Fig. \ref{Fig:staticII}, OSDR outperforms ODR by an almost two order of magnitude smaller error.

\begin{figure}[ht]
\centering
\subfigure[]{
        \includegraphics[width = 0.22\textwidth]{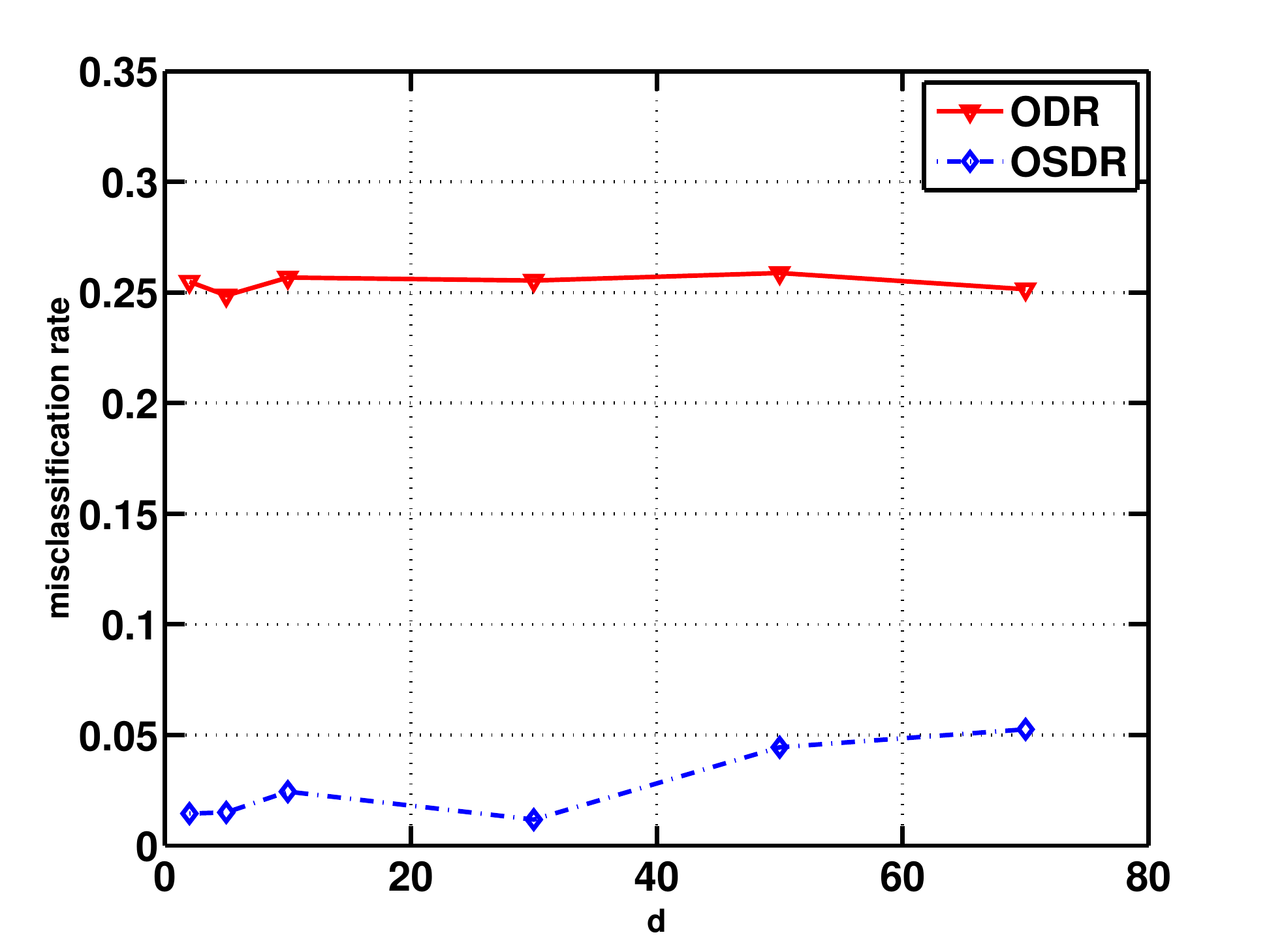} }
\subfigure[]{
        \includegraphics[width = 0.22\textwidth]{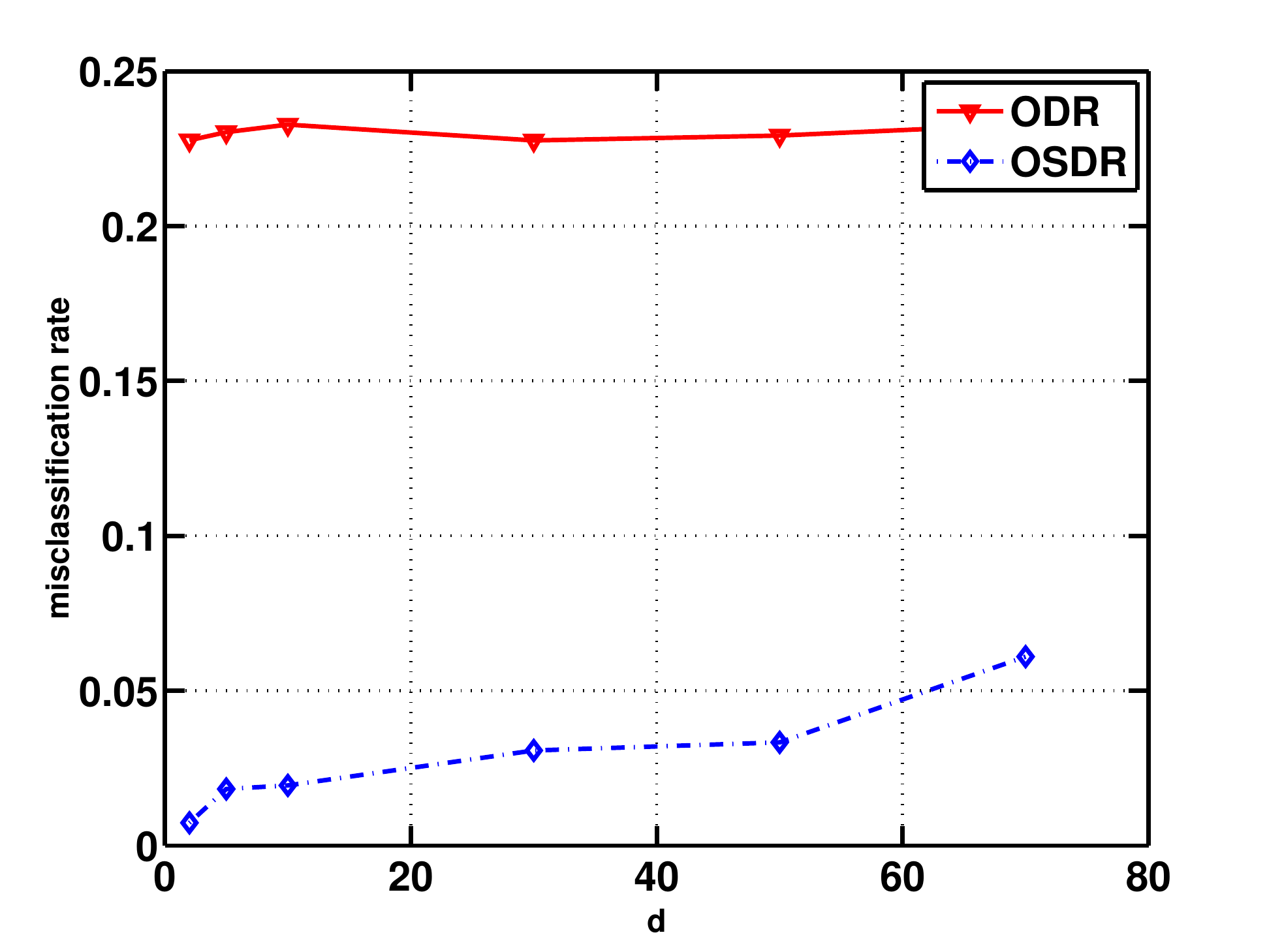} }
\caption{Tracking a static subspace with $D = 100$ and true intrinsic dimension $d^* = 2$. We plot the dimension of the subspace $d$ versus misclassification errors for different aspect ratio of the ellipse in (\ref{ellipse}) 
(a) $r_1/r_2= 3$; (b) $r_1/r_2 = 5$; 
}\label{Fig:staticII}
\end{figure}

\vspace{0.1in}
\noindent{\it Rotating subspaces, logistic regression.} Next we consider tracking a time-varying subspace to demonstrate the capability of OSDR to handle dynamic data. Assume $U_{t} = U_0 R(t)$  with the rotation matrix given by
\begin{equation}
R(t) = \begin{bmatrix} \cos(\alpha_t) & -\sin(\alpha_t) \\
\sin(\alpha_t) & \cos(\alpha_t)\\
\end{bmatrix},
\end{equation}
where $\alpha_t$ is the rotation angle, and  $U_0 \in \mathbb{R}^{p\times 2}$ is a random initial subspace. The vector $\beta_t$ is again generated with entries i.i.d. and lies within an ellipse described by (\ref{ellipse}).
The predictors $x_t$ is generated as the last example. 
%
The rotation angle $\alpha_t$ follows
\begin{equation}
    \alpha_{t}=
   \begin{cases}
   0, &\mbox{if $t \leq$ 500};\\
   \frac{2\pi}{\tau} \cdot \frac{t - 500}{6000 - 500}, & \mbox{if $500 < t \leq 6000$};
   \end{cases}
  \label{eqn:rotationRate}
\end{equation}
where $\tau$ is the rotation speed (smaller $\tau$ corresponds to faster rotation).
The logistic regression coefficient vector is along the shorter axis of the ellipse.  
Fig. \ref{Fig:rotat1} shows $P_e$ for various ration speed, where, again OSDR significantly outperforms ODR.

\begin{figure}[ht]
\centering
\subfigure[]{
        \includegraphics[width = 0.22\textwidth]{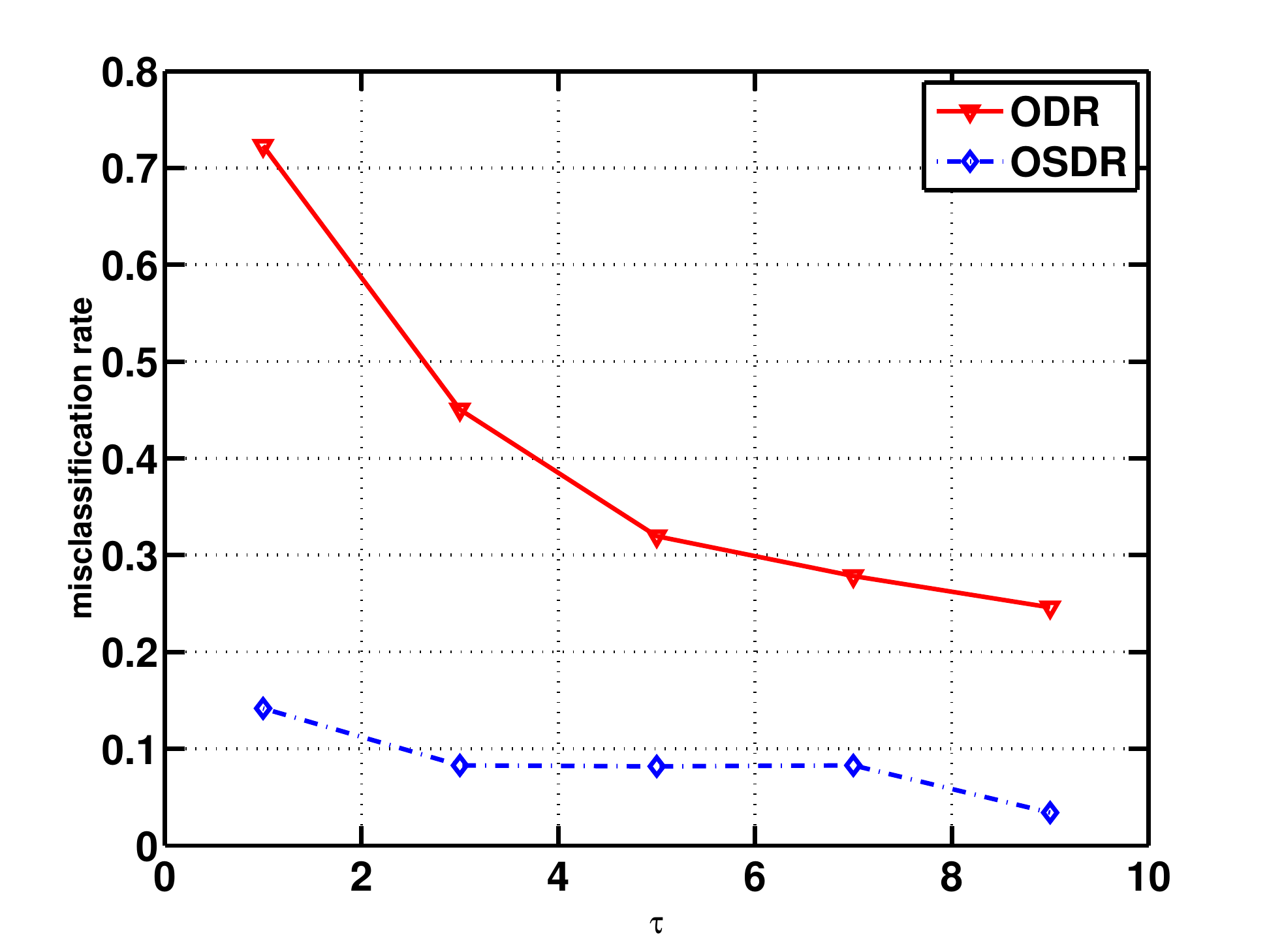} }
\subfigure[]{
        \includegraphics[width = 0.22\textwidth]{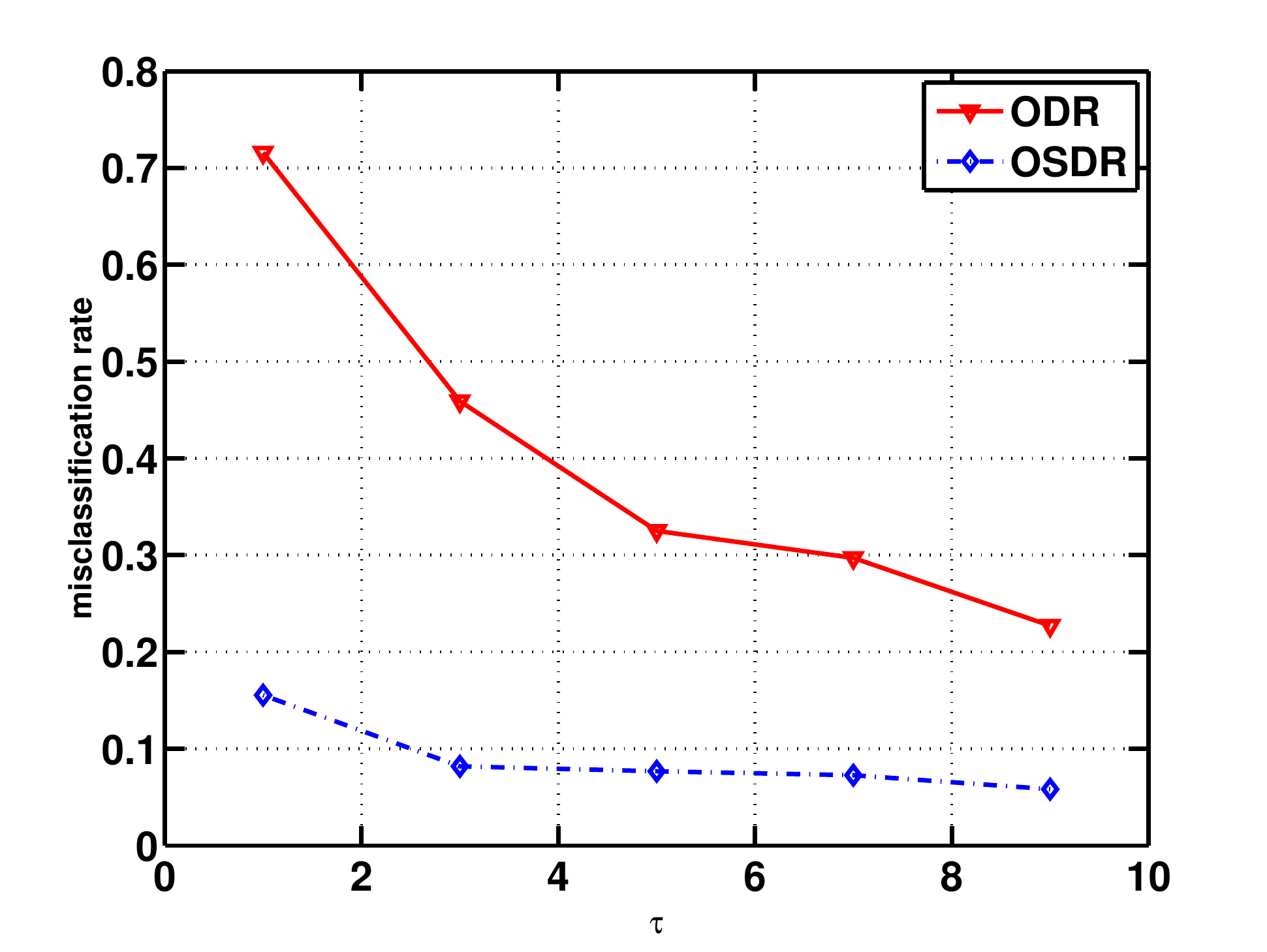} }

\subfigure[]{
        \includegraphics[width = 0.22\textwidth]{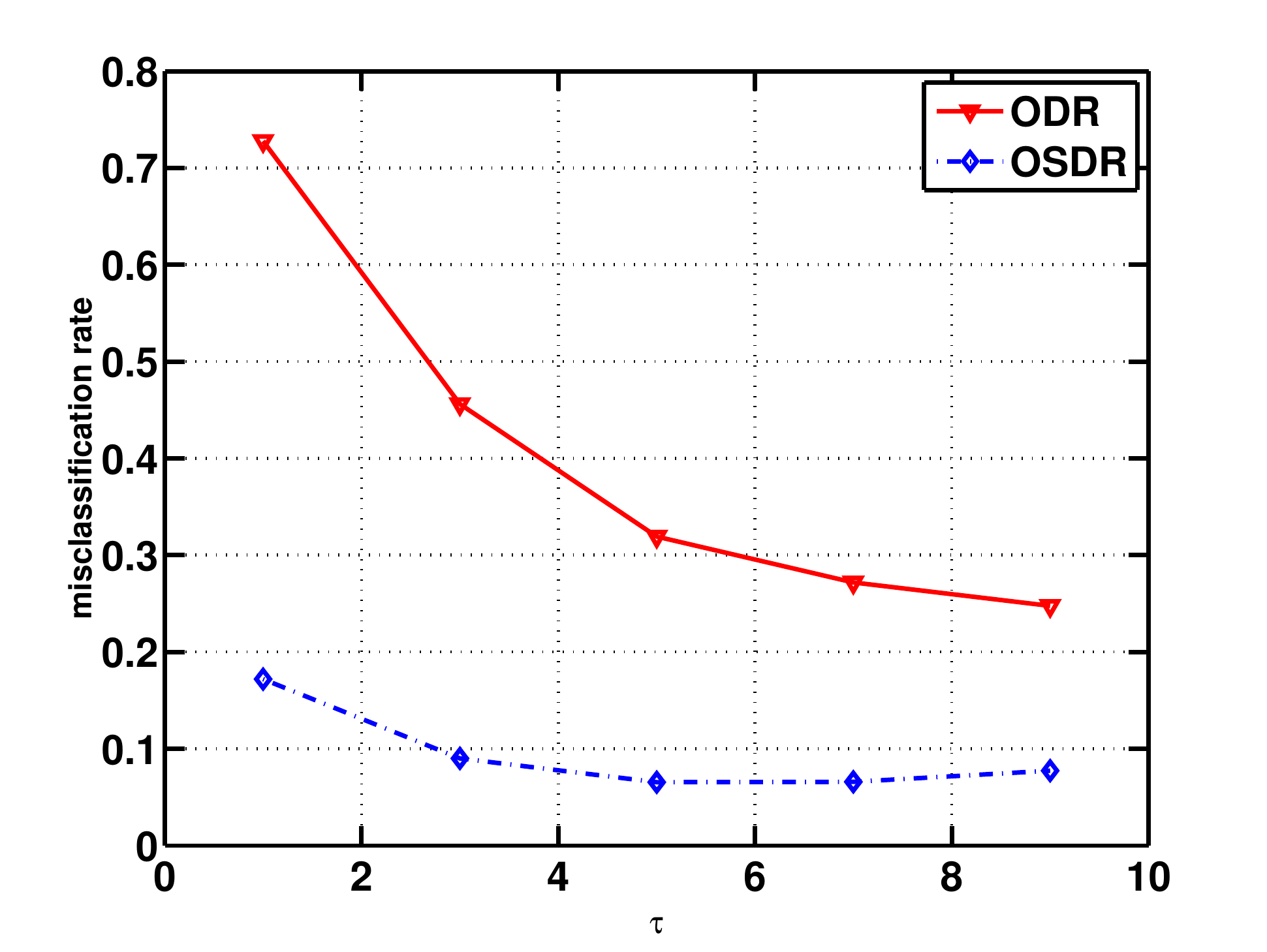} }
\subfigure[]{
        \includegraphics[width = 0.22\textwidth]{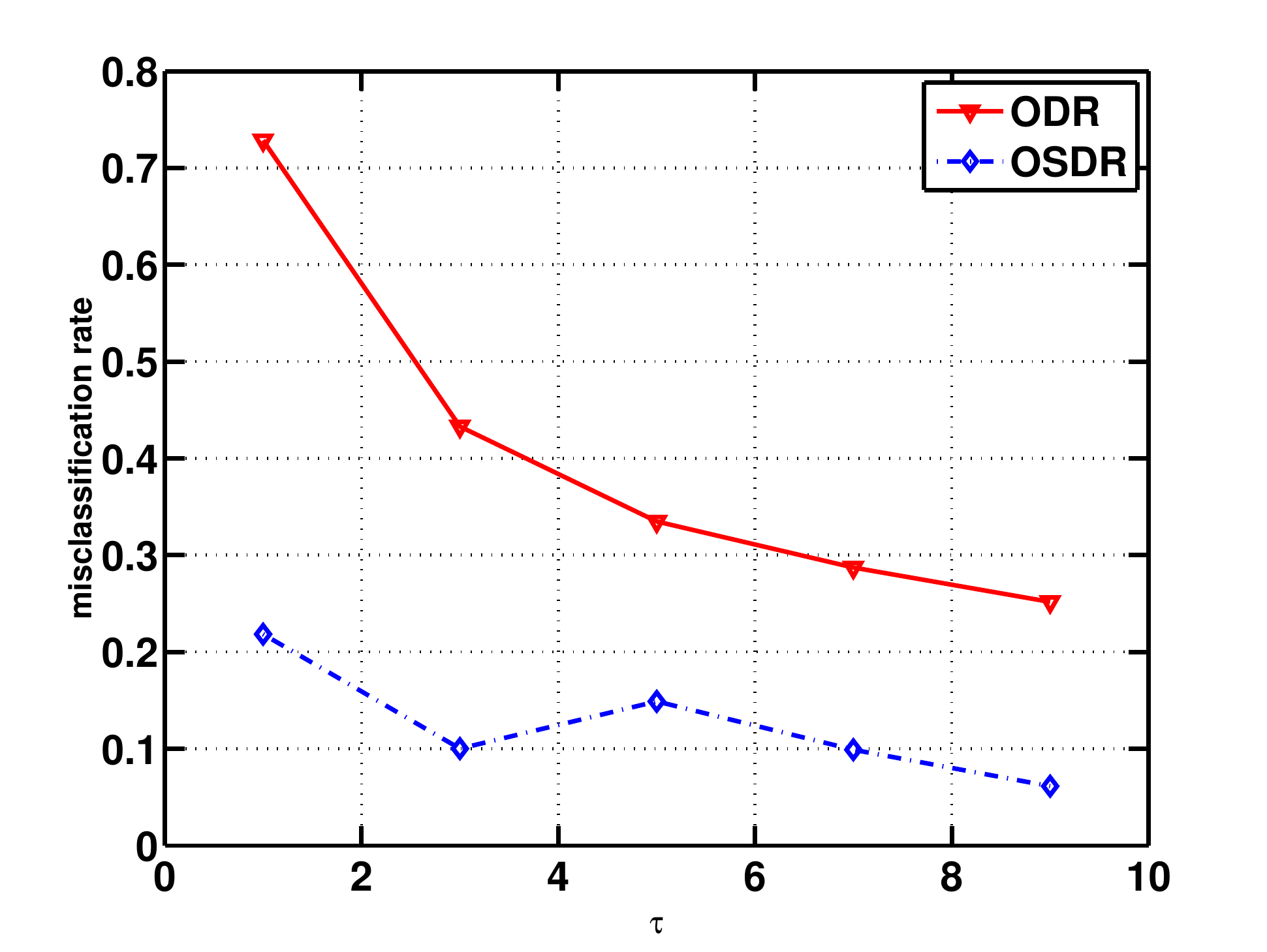} }
\caption{Rotating subspace with aspect ratio of the ellipse $r_1/r_2 = 10$. Rotation ratio $\tau$ versus misclassification rate when (a) $d$ = 30; (b) $d$ = 10; (c) $d$ = 5; (d) $d$ = 2.}
\label{Fig:rotat1}
\end{figure}

\vspace{0.1in}
\noindent{\it Static subspace, linear regression.} The third simple example compares the performance of OSDR with ODR in the linear regression setting. The setup is similar to that of tracking a static subspace, except that the response variable $y_t$ is generated through a linear regression model
$
y_t = c\transpose U \beta + b + \epsilon_t,
$
with $D = 2$, $d = 1$, $c=[c_1, c_2]^\top$ and $\epsilon_t \sim \mathcal{N}(0, \delta^2)$ with $\delta^2 =10^{-3}$.
We use the  rooted mean squared error (RMSE) on the test data as our performance metric, which is the square root of the averaged square error on the predicted $\hat{y}$ differs from true $y$.
Fig. \ref{Fig:linear} shows the RMSE associated with OSDR is significantly lower than that of ODR.

\begin{figure}[ht]
\centering
\subfigure[]{
        \includegraphics[width = 0.22\textwidth]{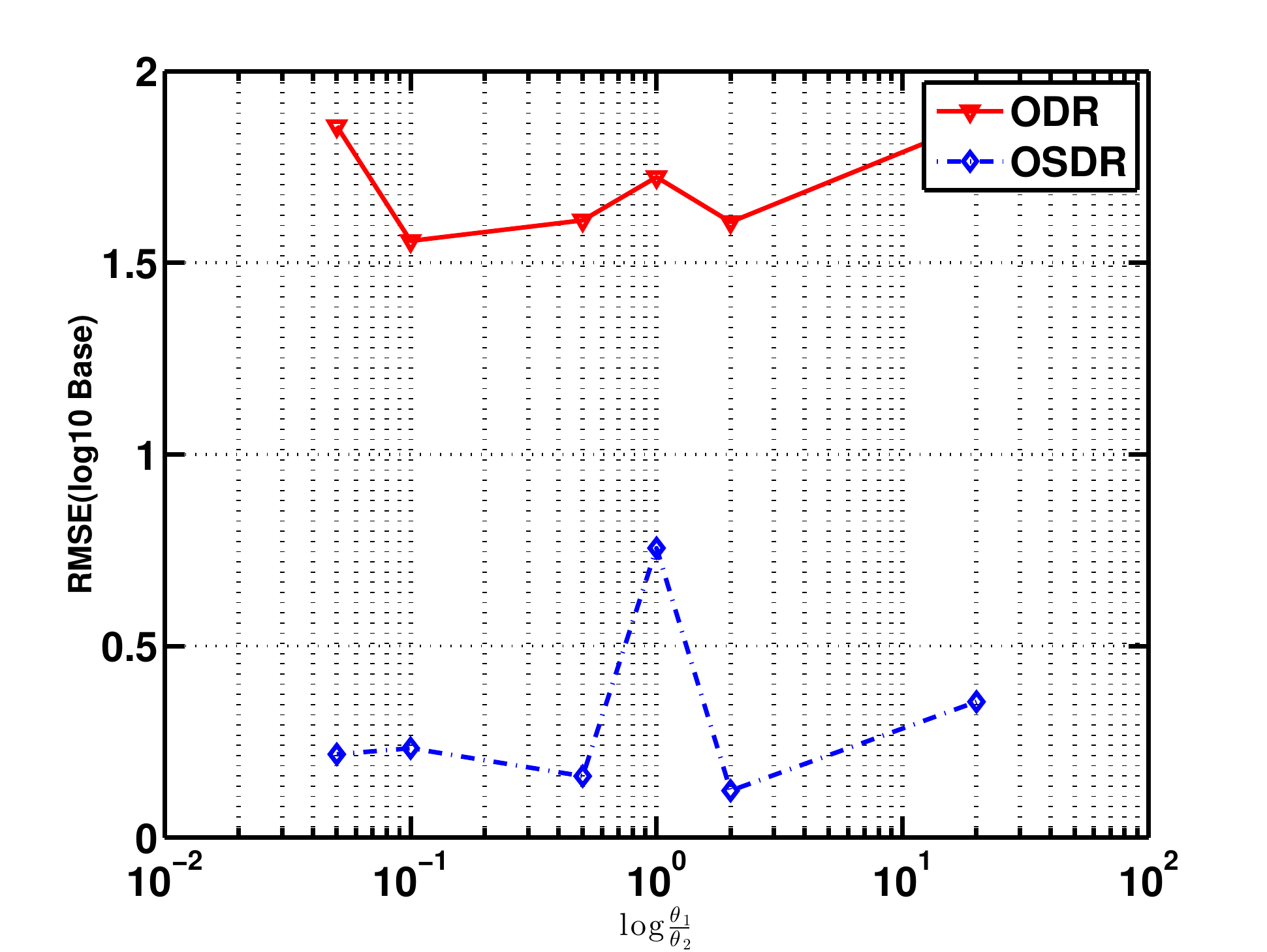} }
\subfigure[]{
        \includegraphics[width = 0.22\textwidth]{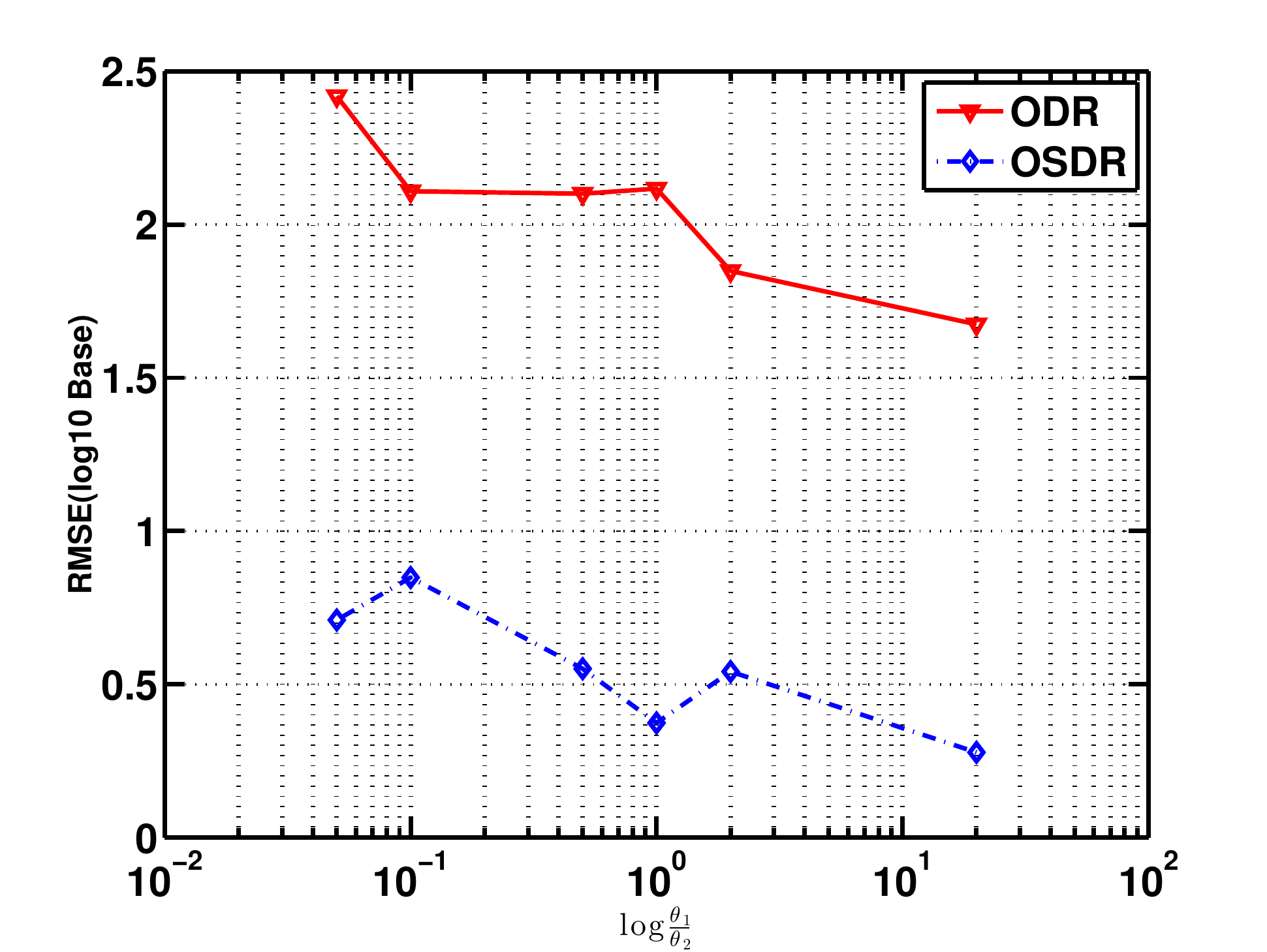} }
\caption{Static subspace, OSDR linear regression: RMSE versus  $\log(c_1 / c_2)$ versus $\log(\mbox{RMSE})$ when (a)  $r_1 / r_2 = 1$; (b) $r_ 1/ r_2 = 2$.}
\label{Fig:linear}
\vspace{-0.2in}
\end{figure}

\begin{figure}[ht]
\centering
\subfigure[Type I Results]{
        \includegraphics[width = 0.22\textwidth]{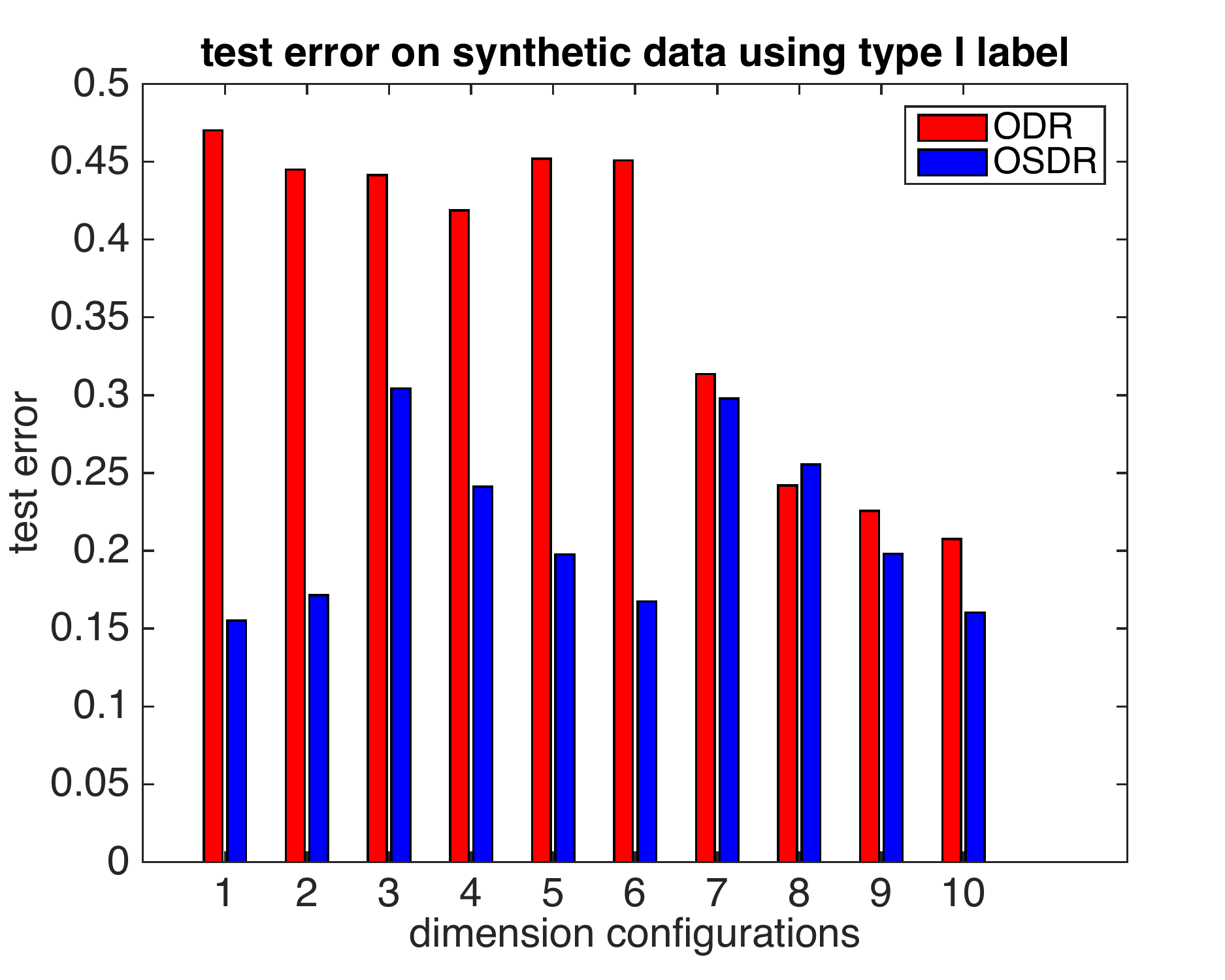} }
\subfigure[Type II Results]{
        \includegraphics[width = 0.22\textwidth]{bar_plot_type_I.pdf} }
\caption{Synthetic classification dataset. (a) test errors $P_e$ under type I; (b) test errors $P_e$ under type II. The ten $(D, d)$ pairs are (2, 1); (3, 2); (10, 2); (10, 4); (10, 6); (10, 8); (50, 10); (50, 20); (50, 30); (50, 40), and have x-labels from 1 to 10 in the figure, correspondingly. }
	\label{Fig:synthetic_results}
\vspace{-0.2in}
\end{figure}

%


%
%
%
%

\subsection{Synthetic classification dataset}

We consider a larger scale example by generating a synthetic dataset for two-class classification. The dataset has $N = 10000$ samples $\{x_n\} \in R^D, n = 1, \ldots, N$. This generating process guarantees that the eigenvalues of the covariance matrix for the dataset are picked from $D, D-1, \ldots, 1$ (This is done by first generate a random subspace, then project the random generated data into this subspace, and also scale the dimensions to have the required eigenvalues). 
An example when $D=3$ are shown in Fig.~\ref{fig:bad_idea}.

	The labels of the data are obtained as follows: Suppose the covariance matrix of data is $X$. We first find the eigenvalues $[\lambda_1, \lambda_2, \ldots, \lambda_D]$ and the corresponding eigenvectors $[v_1, v_2, \ldots, v_D]$ where $\lambda_1 \geq \lambda_2 \geq \ldots \geq \lambda_D$; then select an eigenvector $v_p, p \in [1, \ldots, D]$, and project the data onto this vector. After sorting the projected values, we label the first half as positive, and last half as negative. Consider two settings for selecting $p$: type-I: $p = d + 1$; type-II: $p = d + (D - d) / 2$. Clearly, type-II will be harder, since the corresponding variance of projected data will be smaller. 
	
	We use half of the data for training, and another half for testing. The tuned learning rate $\mu \in [10^{-2}, 10^{-3}, 10^{-4}]$ and $\eta \in [10^{-2}, 10^{-3}, 10^{-4}]$ for both ODR and OSDR. The mean accuracy $P_e$ are reported after 10 rounds of experiments for each setting. Besides different types of labelling directions, we also evaluated different combinations of $(D, d)$ pairs, as shown in the Fig.~\ref{Fig:synthetic_results}. The results show that  OSDR outperforms the SDR (baseline) significantly. This is expected, as the first $d$ leading directions of principle components contain nothing about the label information; therefore, the unsupervised ODR will fail almost surely. The simple example shown in Fig.~\ref{fig:bad_idea} proves that OSDR can identify the correct direction for projecting data. 
	
\subsection{Union-of-subspaces combined with random dot product model}\label{eg:union}

	We generate an example for interaction of two nodes with features $\beta_1$ and $\beta_2$ through a random dot product graph model defined over al union-of-subspaces, as illustrated in Fig. \ref{tree}. There are three sets which correspond to the three leaf nodes in the tree. Each node in the tree is associated with a subset lies on a subspace and also  a logistic regression coefficient vector. At each time, two predictor vectors $x_{1, t}$ and $x_{2, t}$ of $100$ dimensions are observed (which may belong to the same subset or different subsets) and their interaction $y_t$ is generated through logistic regression model that depending on their inner product. In this example, we also assume there are missing data: only $40\%$ of the samples are observed and the variance of the noise is 0.01. The subsets in the tree are also time varying.  The subspace associated with the root node is a random orthogonal matrix that rotates: $U_{1,t} = \exp(Rt)$ with $R$ being a random per-symmetric matrix that determines the rotation. The children nodes of the root node are slight rotation of the subspace of their parent node: $U_{2,1,t} = \exp(R) U_{1,1,t}$, $U_{2,2,t} = \exp(-R) U_{1,1,t}$, $U_{3,1,t} = \exp(R/2) U_{2,1,t}$, $U_{3,2,t} = \exp(-R/2)U_{2,1,t}$. Results in Table \ref{table:1} shows that OSDR outperforms the conventional online logistic regression (which does not perform dimension reduction and ignores the tree structure). 

\begin{table}[h]
\begin{center}
\caption{Comparison of $P_e$ for data generated from a union-of-subspaces combined with random dot product model.}
\begin{tabular}{|c|c|}
\hline
  Online   & Hierarchical  \\
 logistic regression  & OSDR \\\hline
 0.2133 & 0.1440\\\hline
\end{tabular}
\label{table:1}
\end{center}
\vspace{-0.2in}
\end{table}

\subsection{Real-data experiments}

\noindent{\it USPS dataset.} We test OSDR logistic regression on the well-known USPS dataset of handwritten digits. The dataset contains a training set with 7291 samples, and a test set with 2007 samples. The digits are downscaled to $16 \times 16$ pixels. To demonstrate a online setting, we read one picture at a time with the task of classifying whether a digit is ``2'' or not. Again, we first use vectored image as $x_t$, predict label $\hat{y}_t$, and then reveal the true label followed by using ($x_t, y_t$) to update the logistic regression model.
Fig. \ref{Fig:usps} demonstrates OSDR compared favorably with ODR for this task. The improvement is not significant; which may be due to the fact that in this case the dominant eigenvectors captures most useful features for classification already. 

\vspace{0.1in}
\noindent{\it ``Boy or girl'' classification.} We next perform a ``boy or girl'' classification task on an image set we collected from  college students enrolled in a class. This dataset contains 179 grey scaled images, where 116 of them are from boys. Each picture has the resolution of 65x65, thus the original space has a dimension $D=4225$. We train both algorithms from raw images. For the parameter search, we tune the learning rate $\mu \in [10^{-2}, 10^{-3}, 10^{-4}]$ and $\eta \in [10^{-2}, 10^{-3}, 10^{-4}]$ for both algorithms. Experiments are carried out with different settings of $d \in [10, 50, 100, 150]$. The mean test error of 5-fold cross validation on this dataset are reported in Table~\ref{table:boys_girls_exp}, for each configuration.

\begin{figure}[t]
  \centering
  \includegraphics[width=0.22\textwidth]{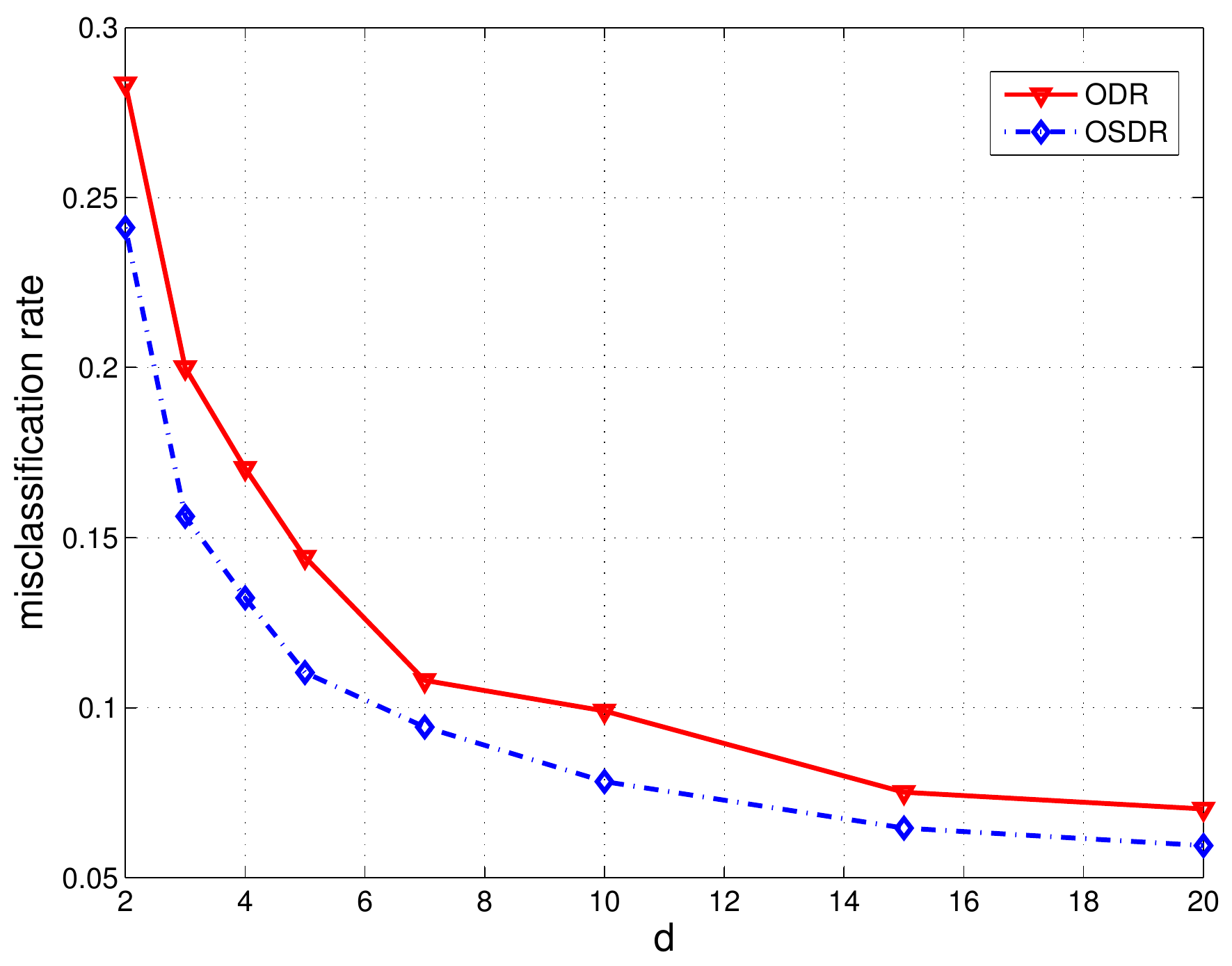}
  \caption{USPS digits recognition. Dimension of the subspace $d$ versus misclassification rate $P_e$ for OSDR and ODR, respectively.}
  \vspace{-0.15in}
  \label{Fig:usps}
\end{figure}

\begin{table}[h]
\begin{center}
\caption{Average test error after 5-fold cross validation on Boy-vs-Girl dataset. }
\begin{tabular}{|c|c|c|c|c|}
\hline
	Method & $d=10$ & $d=50$ & $d=100$ & $d=150$ \\
\hline
	ODR & 0.4800 & 0.4629 & 0.4517 & 0.4229 \\
\hline
	OSDR & 0.3314 & 0.2343 & 0.2171 & 0.1714 \\
\hline
\end{tabular}
\label{table:boys_girls_exp}
\end{center}
\end{table}

    For this dataset, the OSDR algorithm significantly outperforms ODR. 
To gain a better understanding for its good performance, we examine the subspace generated from the experiment. We first visualize the top 7 vectors in the basis of subspace $U$ for OSDR and ODR, respectively. We reshape each vector into a image and, hence, this displays the so-called ``eigenface''. Note that the eigenfaces generated by the unsupervised (corresponding to online PCA) and the supervised subspace learning are very different. 
The online PCA keeps some facial details, while the OSDR algorithm is getting some vectors that are hard to explain visually, but may actually captured details that are more important for telling apart boys and girls.
    
    We further examine the average image of reconstructed faces $x = U\beta$ in Fig.~\ref{Fig:recon_mean}. 
    We compute the average reconstructed image of boys and girls separately, so as to evaluate the discriminative ability of both algorithms. We can see the unsupervised subspace tracking (online PCA) obtained many details of the facial attributes, and the reconstructions of boys and girls have little differences, which makes it hard to distinguish the two genders. So in this case, the unsupervised algorithm fails because of the lack of supervised information. In contrast, the supervised OSDR extracts two very different ``average'' faces for the boy and the girl, respectively, although these average faces do not directly reflect any facial detail. Interestingly, from the contour we learned that, the most discriminative attribute learned by OSDR is the hair (see the dark part in Fig.~\ref{Fig:recon_mean} (d) around shoulder of the girls). This is a straightforward and efficient feature for distinguishing boys and girls. This example clearly demonstrate that OSDR focus on extracting the component that differentiate two classes, and hence, is quite suitable for dimensionality reduction in classification problems.    
    
    

\begin{figure}
\begin{center}
\includegraphics[width = 1\linewidth]{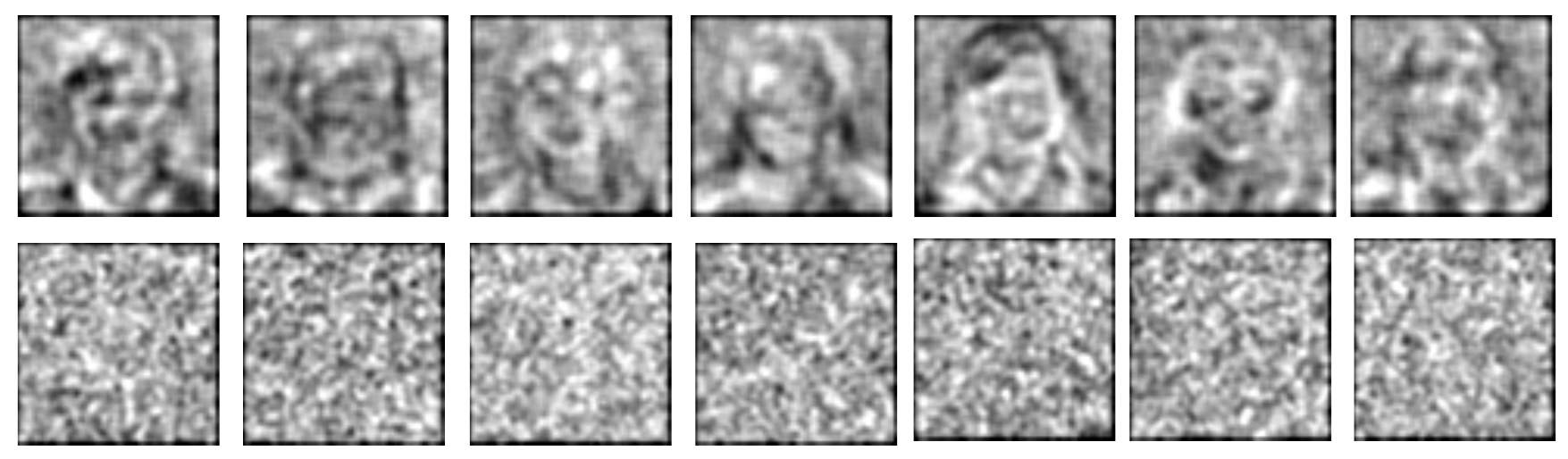}
\end{center}
	\caption{Visualization of top 7 basis in subspace $U$: ``eigenfaces''. Images are processed via blurring and contrast enhancement for better visualization. The first row corresponds to the basis obtained from ODR (baseline) algorithm, while the second row consists of basis from the OSDR algorithm.}
	\label{Fig:basis}
\end{figure}

\begin{figure}
	\centering
	\subfigure[baseline boys]{
		\includegraphics[scale=0.9]{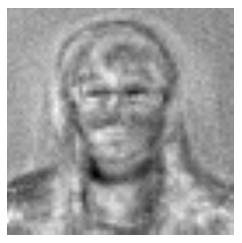}
	}\hspace*{-0.5em}
	\subfigure[baseline girls]{
		\includegraphics[scale=0.9]{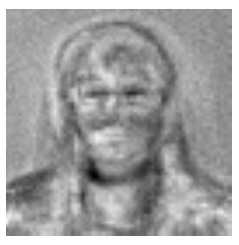}
	}\hspace*{-0.5em}
	\subfigure[OSDR boys]{
		\includegraphics[scale=0.9]{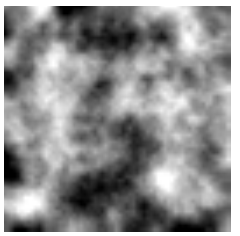}
	}\hspace*{-0.5em}
	\subfigure[OSDR girls]{
		\includegraphics[scale=0.9]{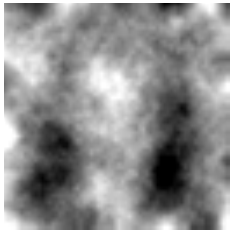}
	}
	\caption{Mean reconstructed image for boys and girls. (a) and (b) are showing results obtained from ODR algorithm, while (c) and (d) are showing results created by OSDR. Images are processed via blurring and contrast enhancement for better visualization.
    }
	\label{Fig:recon_mean}
\end{figure}

%

\bibliographystyle{plain}
\bibliography{bib}

\begin{thebibliography}{10}

\bibitem{blindSource2001}
M.~Zibulevsky and B.~A. Pearlmutter, ``Blind source separation by sparse
  decomposition in a signal dictionary,'' {\em Neural computation}, vol.~13,
  no.~4, pp.~863--882, 2001.

\bibitem{DictionarySapiro09}
J.~Marial, F.~Bach, J.~Ponce, and G.~Sapiro, ``Online dictionary learning for
  sparse coding,'' in {\em Proc. 26th Int. Conf. Machine Learning (ICML)},
  2009.

\bibitem{wrightDictionary2012}
D.~A. Spielman, H.~Wang, and J.~Wright, ``Exact recovery of sparsely-used
  dictionaries,'' in {\em Proc. 25th Annal Conf. Learning Theory}, 2012.

\bibitem{GROUSE1}
L.~Balzano, R.~Nowak, and B.~Recht, ``Online identification and tracking of
  subspaces from highly incomplete information,'' in {\em Communication,
  Control, and Computing (Allerton), 2010 48th Annual Allerton Conference On},
  2010.

\bibitem{robustPCA13}
J.~Feng, H.~Xu, and S.~Yan, ``Online robust {PCA} via stochastic
  optimization,'' in {\em Neural Information Processing Systems Foundation
  (NIPS)}, 2013.

\bibitem{WangTu13}
B.~Wang and Z.~Tu, ``Sparse subspace denoising for image manifolds,'' in {\em
  2013 IEEE Conf. Computer and Vision Pattern Recognition (CVPR)}, 2013.

\bibitem{subspaceDim2014}
D.~Arpit, I.~Nwogu, and V.~Govindaraju, ``Dimensionality reduction with
  subspace structure preservation,'' in {\em Neural Information Processing
  Systems Foundation (NIPS)}, 2014.

\bibitem{BiswasBasu2011}
K.~K. Biswas and S.~K. Basu, ``Gesture recognition using microsoft kinect,'' in
  {\em Automation, Robotics and Applications (ICARA), 2011 5th Int. Conf.},
  2011.

\bibitem{KinectHand2012}
Y.~Li, ``Hand gesture recognition using kinect,'' in {\em Software Engineering
  and Service Science (ICSESS), 2012 IEEE 3rd Int. Conf. on}, 2012.

\bibitem{GROUSE2}
L.~Balzano and S.~J. Wright, ``On grouse and incremental svd,'' in {\em IEEE
  5th International Workshop on Computational Advances in Multi-Sensor Adaptive
  Processing (CAMSAP)}, 2013.

\bibitem{GROUSEproof}
L.~Balzano and S.~J. Wright, ``Local convergence of an algorithm for subspace
  identification from partial data.,'' {\em Foundations of Computational
  Mathematics}, pp.~1--36, Oct. 2014.

\bibitem{PETRELS}
Y.~Chi, Y.~C. Eldar, and R.~Calderbank., ``Petrels: Parallel subspace
  estimation and tracking using recursive least squares from partial
  observations,'' {\em IEEE Trans. on Signal Processing}, vol.~61, pp.~5947 --
  5959, 2013.

\bibitem{MOUSSE2013}
Y.~Xie, J.~Huang, and R.~Willett, ``Change-point detection for high-dimensional
  time series with missing data,'' {\em IEEE Journal of Selected Topics in
  Signal Processing (J-STSP)}, vol.~7, pp.~12--27, Feb. 2013.

\bibitem{logistic2012}
Y.~Xie and R.~Willett, ``Online logistic regression on manifolds,'' in {\em
  IEEE Int. Conf. Acoustics, Speeches and Sig. Proc. (ICASSP)}, 2012.

\bibitem{SIR1991}
K.-C. Li, ``Sliced inverse regression for dimension reduction,'' {\em J.
  American Stat. Association}, vol.~86, pp.~316--327, 1991.

\bibitem{CookForzani2009}
R.~D. Cook and L.~Forzani, ``Likelihood-based sufficient dimension reduction,''
  {\em J. American Stat. Association}, pp.~197--208, 2009.

\bibitem{NilssonShaJordan2007}
F.~S. J.~Nilsson and M.~I. Jordan, ``Regression on manifolds using kernel
  dimension reduction,'' {\em Proc. 24th Int. Conf. Machine Learning (ICML)nd
  Int. Conf. Machine Learning (ICML)}, 2007.

\bibitem{edelman1998geometry}
A.~Edelman, T.~A. Arias, and S.~T. Smith, ``The geometry of algorithms with
  orthogonality constraints,'' {\em SIAM journal on Matrix Analysis and
  Applications}, vol.~20, no.~2, pp.~303--353, 1998.

\bibitem{NickelThesis2006}
C.~L.~M. Nickel, {\em Random dot product graph: a model for social networks}.
\newblock PhD thesis, University of Maryland, 2006.

\bibitem{YoungScheinerman2007}
S.~J. Young and E.~R. Scheinerman, ``Random dot product graph models for social
  networks,'' {\em Algorithms and Models for the Web-Graph: Lecture Notes in
  Computer Science Volume}, vol.~4863, pp.~138--149, 2007.

\bibitem{randomDot2}
E.~R. Scheinerman and K.~Tucker, ``Modeling graphs using dot product
  representations,'' {\em Computational Statistics}, vol.~25, pp.~1--16, 2010.

\bibitem{high_rank_MC12}
B.~Eriksson, L.~Balzano, and R.~Nowak, ``High rank matrix completion,'' in {\em
  Proc. of Intl. Conf. on Artificial Intell. and Stat}, 2012.

\bibitem{SSP12}
Y.~Xie, J.~Huang, and R.~Willett, ``Multiscale online tracking of manifolds,''
  in {\em IEEE Statistical Signal Processing Workshop (SSP)}, 2012.

\bibitem{AllardChenMaggioni2011}
W.~Allard, G.~Chen, and M.~Maggioni, ``Multi-scale geometric methods for data
  sets {II}: Geometric multi-resolution analysis,'' {\em App. and Comp.
  Harmonic Ana.}, vol.~32, pp.~435 -- 462, May 2011.

\bibitem{yang1995projection}
B.~Yang, ``Projection approximation subspace tracking,'' {\em Signal
  Processing, IEEE Transactions on}, vol.~43, no.~1, pp.~95--107, 1995.

\bibitem{niesen2009adaptive}
U.~Niesen, D.~Shah, and G.~W. Wornell, ``Adaptive alternating minimization
  algorithms,'' {\em Information Theory, IEEE Transactions on}, vol.~55, no.~3,
  pp.~1423--1429, 2009.

\bibitem{epsilon_net_2014}
N.~Alon, T.~Lee, A.~Shraibman, and S.~Vempala, ``The approximate rank of a
  matrix and its algorithmic applications,'' {\em Proceedings of the
  forty-fifth annual ACM symposium on Theory of computing}, 2013.

\bibitem{PlanThesis2011}
Y.~Plan, {\em Compressed sensing, sparse approximation, and low-rank matrix
  estimation}.
\newblock PhD thesis, California Institute of Technology, 2011.

\end{thebibliography}

\end{document}